\documentclass[twoside,11pt]{article}

\usepackage{jmlr2e}
\usepackage{amssymb,amsbsy,amsfonts,amsmath,xspace}
\usepackage{mathrsfs}
\usepackage{graphicx}
\usepackage{caption}
\usepackage{setspace}
\usepackage{comment}
\usepackage{enumitem}
\usepackage{subcaption}
\usepackage{JLMR_micro}
\usepackage{multirow}
\usepackage{epstopdf}
\usepackage{epsfig}

\usepackage{url}
\usepackage[toc,page]{appendix}
\usepackage{float}
\usepackage{natbib}
\usepackage{color}
\usepackage{verbatim}

\makeatletter
\newcommand{\printfnsymbol}[1]{%
  \textsuperscript{\@fnsymbol{#1}}%
}
\makeatother

\def\text#1{\mbox{\rm #1}}

\def\F{\mathrm}

\newcommand{\ve}{\mathrm{vec}}

\usepackage{lastpage}
\jmlrheading{21}{2020}{1-\pageref{LastPage}}{7/19; Revised
4/20}{4/20}{19-563}{Tianxi Li, Cheng Qian, Elizaveta Levina and Ji Zhu}
\ShortHeadings{High-dimensional Gaussian graphical models on network-linked data}{Li, Qian, Levina and Zhu}

\begin{document}

\title{High-dimensional Gaussian graphical models on network-linked data}

\author{\name Tianxi Li\thanks{Authors with equal contribution}  \email tianxili@virginia.edu\\
       \addr Department of Statistics\\
       University of Virginia\\
       Charlottesville, VA 22904, USA
      \AND
\name Cheng Qian\printfnsymbol{1} \email qianc@seu.edu.cn\\
       \addr School of Mathematics\\
       Southeast University\\
       Nanjing, Jiangsu 211189, China
        \AND
\name Elizaveta Levina \email elevina@umich.edu
       \AND 
\name Ji Zhu \email jizhu@umich.edu \\
       \addr Department of Statistics\\
       University of Michigan\\
       Ann Arbor, MI 48109, USA
       }

\editor{Jie Peng}

\maketitle

\begin{abstract}
Graphical models are commonly used to represent conditional
dependence relationships between variables.   There are multiple
methods available for exploring them from high-dimensional data, but
almost all of them rely on the assumption that the observations are
independent and identically distributed.  At the same time,
observations connected by a network are becoming increasingly common,
and tend to violate these assumptions.    Here we
develop a Gaussian graphical model for observations connected by a
network with potentially different mean vectors, varying smoothly over
the network.     We propose an efficient estimation algorithm and
demonstrate its effectiveness on both simulated and real data,
obtaining meaningful and interpretable results on a statistics coauthorship network.      We also prove that our method estimates both the inverse covariance matrix and the corresponding graph structure correctly under the assumption of network ``cohesion'', which refers to the empirically observed phenomenon of network neighbors sharing similar traits.  
\end{abstract}


\section{Introduction}\label{sec:intro}
Network data represent information about relationships (edges) between units (nodes), such as friendships or collaborations, and are often collected together with more ``traditional'' covariates that describe one unit.   In a social network, edges may represent friendships between people (nodes), and traditional covariates could be their demographic characteristics such as gender, race, age, and so on.     Incorporating relational information in statistical modeling tasks focused on ``traditional'' node covariates should improve performance, since it offers additional information, but most traditional multivariate analysis methods are not designed to use such information.   In fact,  most such methods for regression, clustering, density estimation and so on tend to assume the sampled units are homogeneous, typically independent and identically distributed (i.i.d.), which is unlikely to be the case for units connected by a network.   While there is a fair amount of  work on incorporating such information into specific settings \citep{manski1993identification, lee2007identification, yang2011multi,raducanu2012supervised, vural2016out},   work on extending standard statistical methods to network-linked data has only recently started appearing, for example,  \cite{li2016prediction} for regression, \cite{tang2013universally} for classification, and \cite{yang2013community}, \cite{binkiewicz2014covariate}
for clustering.     Our goal in this paper is to develop an analog to the widely used Gaussian graphical models for network-linked data which takes advantage of this additional information to improve performance when possible.

Graphical models are commonly used to represent independence relationships between random variables, with each variable corresponding to a node, and edges representing conditional or marginal dependence between two random variables.    Note that a graphical model is a graph connecting variables, as opposed to the  networks discussed above, which are graphs connecting observations.    Graphical models have been widely studied in statistics and machine learning and have applications in bioinformatics, text mining and causal inference, among others.    The Gaussian graphical model belongs to the family of undirected graphical models, or Markov random fields, and assumes the variables are jointly Gaussian.  Specifically, the conventional Gaussian graphical model for a data matrix $X \in \bR^{n\times p}$ assumes that the rows $X_{i\cdot}$, $i = 1, \dots, n$,  are independently drawn from the same $p$-variate normal distribution $\ncal(\mu, \Sigma)$.
This vastly simplifies analysis, since for the Gaussian distribution all  marginal dependence information is contained in the covariance matrix, and all  conditional independence information in its inverse.  In particular, random variables $j$ and $j'$ are conditionally independent given the rest if and only if the $(j,j')$-th entry of the inverse covariance matrix $\Sigma^{-1}$ (the precision matrix) is zero.      Therefore estimating the graph for a Gaussian graphical model is equivalent to identifying zeros in the precision matrix, and this problem has been well studied, in both the low-dimensional and the high-dimensional settings.    A pioneering paper by \cite{meinshausen2006high} proposed neighborhood selection, which learns edges by regressing each variable on all the others via lasso, and established good asymptotic properties in high dimensions.    Many penalized likelihood methods have  been proposed as well \citep{yuan2007model, banerjee2008model, rothman2008sparse, d2008first, friedman2008sparse}. In particular, the graphical lasso (glasso) algorithm of \cite{friedman2008sparse} and its subsequent improvements  \citep{witten2011new,hsieh2013bigquic}  are widely used to solve the problem efficiently.

The penalized likelihood approach to Gaussian graphical models assumes the observations are i.i.d., a restrictive assumption in many real-world situations.    This assumption was relaxed in \cite{zhou2010time, guo2011joint} and \cite{danaher2014joint}  by allowing the covariance matrix to vary smoothly over time or across groups, while the mean vector remains constant.  A special case of modeling the mean vector on additional covariates associated with each observation has also been studied \citep{rothman2010sparse, yin2011sparse,lee2012simultaneous,  cai2013covariate, lin2016penalized}.    Neither of these relaxations are easy to adapt to network data, and their assumptions are hard to verify in practice.  
%

In this paper, we consider the problem of estimating a graphical model with heterogeneous mean vectors when a network connecting the observations is available. For example, in analyzing word frequencies in research papers, the conditional dependencies between words may represent certain common phrases used by all authors.   However, since different authors also have different research topics and writing styles, there is individual variation in word frequencies themselves, and the coauthorship information is clearly directly relevant to modeling both the universal dependency graph and the individual means.   We propose a generalization of the Gaussian graphical model to such a setting, where each data point can have its own mean vector but the data points share the same covariance structure.     We further assume that a network connecting the observations is available, and that the mean vectors exhibit network ``cohesion'', a generic term describing the phenomenon of connected nodes behaving similarly,  observed widely in empirical studies and experiments  \citep{fujimoto2012social,haynie2001delinquent,christakis2007spread}.    We develop a computationally efficient algorithm to estimate the proposed Gaussian graphical model with network cohesion, and show that the method is consistent for estimating both the covariance matrix and the graph  in high-dimensional settings under a network cohesion assumption. 
Simulation studies show that our method works as well as the standard Gaussian graphical model in the i.i.d.\ setting, and is effective in the setting of different means with network cohesion, while the standard Gaussian graphical model completely fails.

The rest of the paper is organized as follows. Section~\ref{sec:model} introduces a Gaussian graphical model on network-linked observations and the corresponding two-stage model estimation procedure.   An alternative estimation procedure based on joint likelihood is also introduced, although we will argue that the two-stage estimation is preferable from both the computational and the theoretical perspectives.  Section~\ref{sec:theory1} presents a formal definition of network cohesion and error bounds under the assumption of network cohesion and regularity conditions, showing we can consistently estimate the partial dependence graph and model parameters. Section~\ref{sec:sim} presents simulation studies comparing the proposed method to standard graphical lasso and the two-stage estimation algorithm to the joint likelihood approach.    Section~\ref{sec:app}  applies the method to analyzing  dependencies between terms from a collection of statistics papers' titles and the associated coauthorship network.  Section~\ref{sec:con} concludes with discussion.

\section{Gaussian graphical model with network cohesion}\label{sec:model}

\subsection{Preliminaries} 
We start with setting up notation.  For a matrix $X \in \bR^{n\times p}$, let $X_{\cdot j}$ be the $j$th column and $X_{i\cdot}$  the $i$th row.   By default, we treat all vectors as column vectors. Let $\norm{X}_F = (\sum_{i,j} X_{ij}^2)^{1/2}$ be the Frobenius norm of $X$ and $\norm{X}$ the spectral norm, i.e.,  the largest singular value of $X$.    Further, let $\norm{X}_0 = \#\{(i,j): X_{ij}\ne 0\}$ be the number of non-zero elements in $X$,  $\norm{X}_1 = \sum_{ij}|X_{ij}|$, and  $\norm{X}_{1,\text{off}} = \sum_{i\ne j}|X_{ij}|$. For a square matrix $\Sigma$, let $\tr(\Sigma)$ and $\det(\Sigma)$ be the trace and the determinant of $\Sigma$, respectively, and assuming $\Sigma$ is a covariance matrix, let $r(\Sigma) = \frac{\tr(\Sigma)}{\norm{\Sigma}}$ be its {\em stable rank}.   It is clear that $1\le r(\Sigma) \le p$ for any nonzero covariance matrix $\Sigma$.

While it is common, and not incorrect, to use the terms ``network'' and ``graph'' interchangeably, throughout this paper   ``{\em network}" is used to refer to the {\em observed} network connecting the $n$ observations,   and  ``{\em graph}" refers to the conditional dependence graph of $p$ variables {\em to be estimated}.
 In a network or graph $\gcal$ of size $n$, if two nodes $i$ and $i'$ of $\gcal$ are connected, we write $i\sim_{\gcal}i'$, or  $i \sim i'$ if $\gcal$ is clear from the context.  The adjacency matrix of a graph $\gcal$ is an $n\times n$ matrix $A$ defined by $A_{ii'} = 1$ if $i\sim_{\gcal}i'$ and 0 otherwise.    We focus on undirected networks, which implies the adjacency matrix is symmetric.  Given an adjacency matrix $A$, we define its Laplacian by $L = D-A$ where $D = \F{diag}(d_1, d_2, \cdots, d_n)$ and $d_i = \sum_{i' = 1}^nA_{ii'}$ is the degree of node $i$. A well-known property of the Laplacian matrix $L$ is that, for any vector $\mu \in \bR^n$, 
\begin{equation}\label{eq:LaplacianPenalty-scalar}
\mu^TL\mu = \sum_{i\sim i'}(\mu_i-\mu_{i'})^2.
\end{equation}
We also define a normalized Laplacian $\lcal_s = \frac{1}{\bar{d}}L$ where $\bar{d}$ is the average degree of the network $\gcal$, given by $\bar d = \frac{1}{n}\sum_i d_i$.  We denote the eigenvalues of $\lcal_s$ by $\tau_1 \ge \tau_2 \ge  \cdots \ge \tau_{n-1} \ge \tau_n=0$, and the corresponding eigenvectors by $u_1, \dots, u_n$.

\subsection{Gaussian graphical model with network cohesion (GNC)}

We now introduce the heterogeneous Gaussian graphical model, as a generalization of the standard Gaussian graphical model with i.i.d.\ observations.  Assume the data matrix  $X$ contains $n$ independent observations $X_{i\cdot} \in \bR^p, i=1,2, \cdots, n$.  Each $X_{i\cdot}$ is a random vector drawn from an individual multivariate Gaussian distribution
\begin{equation}\label{eq:GNC-GGM}
X_{i\cdot} \sim \ncal(\mu_i, \Sigma), i=1,2, \cdots, n.
\end{equation}
where $\mu_{i} \in \bR^p$ is a $p$-dimensional vector and $\Sigma$ is a $p\times p$ symmetric positive definite matrix. Let $\Theta = \Sigma^{-1}$ be the corresponding precision matrix and  $M = (\mu_{1}, \mu_{2}, \cdots, \mu_{n})^T$ be the mean matrix, which will eventually incorporate cohesion.     Recall that in the Gaussian graphical model, $\Theta_{jj'} = 0$ corresponds to the conditional independence relationship $x_j\perp x_j' | \{x_k, k\ne j,j'\}$ \citep{lauritzen1996graphical}. Therefore a typical assumption, especially in high-dimensional problems, is that $\Theta$ is a sparse matrix; this both allows us to estimate $\Theta$ when $p > n$, and produces a sparse conditional dependence graph.

Model \eqref{eq:GNC-GGM} is much more flexible than the i.i.d.\ graphical model, and it separates  co-variation caused by individual preference (cohesion in the mean) from universal co-occurrence (covariance).    The price we pay for this flexibility is the much larger number of parameters, and model  \eqref{eq:GNC-GGM} cannot be fitted without additional assumptions on the mean, since we only have one observation to estimate each vector $\mu_i$.   The structural assumption we make in this paper is \emph{network cohesion}, a phenomenon of connected individuals in a social network tending to exhibit similar traits. It has been widely observed  in many empirical studies such as health-related behaviors or academic performance \citep{michell1996peer,haynie2001delinquent,pearson2003drifting}. 
Specifically, in our Gaussian graphical model~\eqref{eq:GNC-GGM}, we assume that connected  nodes in the observed network  have similar mean vectors.   This assumption is reasonable and interpretable in many applications. For instance, in  the coauthorship network example, cohesion indicates coauthors tend to have similar word preferences, which is reasonable since they work on similar topics and share at least some publications.

\subsection{Fitting the GNC model}

The log-likelihood of the data under model \eqref{eq:GNC-GGM} is, up to a constant,
\begin{equation}\label{eq:loglike}
\ell(M, \Theta) = \log\det(\Theta) - \frac{1}{n} \tr(\Theta(X - M)^T(X-M)). 
\end{equation}
A sparse inverse covariance matrix $\Theta$ and a cohesive mean matrix $M$ are naturally incorporated into the following two-stage procedure, which we call  \underline{G}aussian graphical model estimation with \underline{N}etwork \underline{C}ohesion and \underline{lasso} penalty (GNC-lasso).
\begin{algo}[Two-stage GNC-lasso algorithm]\label{algo:twostage}
Input:  a standardized data matrix $X$, network adjacency matrix $A$, tuning parameters $\lambda$ and $\alpha$.
\begin{enumerate}
\item Mean estimation.   Let $L_s$ be the standardized Laplacian of $A$.   Estimated the mean matrix by 
\begin{equation}\label{eq:individualRNC}
\hat{M} = \arg \min_M  \norm{X-M}_F^2 + \alpha \, \tr(M^T\lcal_sM).
\end{equation}
\item Covariance estimation. Let $\hat{S} = \frac{1}{n}(X-\hat{M})^T(X-\hat{M})$ be the sample covariance matrix of $X$ based on $\hat M$.    Estimate the precision matrix  by 
\begin{equation}\label{eq:glassoObj}
\hat{\Theta} = \arg\min_{ \Theta \in \bS^{n}_+}\log\det(\Theta) - \tr(\Theta\hat{S}) - \lambda\norm{\Theta}_{1,\text{off}}.
\end{equation}
\end{enumerate}
\end{algo}
The first step is a penalized least squares problem, where the penalty can be written as
\begin{equation}\label{eq:LaplacianPenalty-vec}
 \tr(M^T\lcal_sM) = \sum_{i\sim i'}\norm{\mu_i-\mu_{i'}}^2.
 \end{equation}
This can be viewed as a vector version of the Laplacian penalty used in variable selection \citep{li2008network,li2010variable, zhao2016significance} and regression problems \citep{li2016prediction} with network information.  It penalizes the difference between mean vectors of connected nodes, encouraging cohesion in the estimated mean matrix.   Both terms in \eqref{eq:individualRNC} are separable in the $p$ coordinates and the least squares problem has a closed form solution, 
\begin{equation}\label{eq:LaplacianSmoothing}
\hat{M}_{\cdot j} = (I_n+\alpha\lcal_s)^{-1}X_{\cdot j}, ~~j = 1,2, \cdots, p.
\end{equation}
In practice, we usually need to compute the estimate for a sequence of $\alpha$ values, so we first calculate the eigen-decomposition of $\lcal_s$ and then obtain each $(I+\alpha\lcal_s)^{-1}$ in linear time.  In most applications, networks are very sparse, and taking advantage of sparsity and the symmetrically diagonal dominance of $\lcal_s$ allows to compute the eigen-decomposition very efficiently \citep{cohen2014solving}. Given $\hat{M}$, criterion \eqref{eq:glassoObj} is a graphical lasso problem that uses the lasso penalty \citep{tibshirani1996regression} to encourage sparsity in the estimated precision matrix, and can be solved by the glasso algorithm \citep{friedman2008sparse} efficiently or any of its variants, later significantly improved further by \cite{witten2011new} and \cite{hsieh2014quic,hsieh2013big}.


\subsection{An alternative: penalized  joint likelihood}\label{secsec:joint}
An alternative and seemingly more natural approach is to maximize a penalized log-likelihood to estimate both $M$ and $\Theta$ jointly as 
\begin{align}\label{eq:obj1}
(\hat \Theta, \hat M) = \arg \max_{\Theta, M}~\log\det(\Theta) &- \frac{1}{n}\tr(\Theta(X - M)^T(X-M)) - \lambda\norm{\Theta}_{1,\text{off}} - \frac{\alpha}{n}\tr(M^T\lcal_sM).
\end{align}
The objective function is bi-convex and the optimization problem can be solved by alternately optimizing over $M$ with fixed $\Theta$ and then optimizing over $\Theta$ with fixed $M$ until convergence.    We refer to this method as iterative GNC-lasso.     Though this strategy seems more principled in a sense, we implement our method with the two-stage algorithm,  for the following reasons.  

First,  the computational complexity of the iterative method based on joint likelihood is significantly higher, and it does not scale well in either $n$ or $p$.    This is because when $\Theta$ is fixed and we need to maximize over $M$, all $p$ coordinates are coupled in the objective function, so the scale of the problem is $np\times np$.   Even for moderate $n$ and $p$,  solving this problem requires either a large amount of memory or applying Gauss-Seidel type algorithms that  further increase the number of iterations.   This problem is exacerbated by the need to select two tuning parameters $\lambda$ and $\alpha$ jointly, because, as we will discuss later, they are also coupled.  

More importantly,  our empirical results show that the iterative estimation method does not improve on the two-stage method (if it does not slightly hurt it).    The same phenomenon was observed empirically by \cite{yin2013adjusting} and \cite{lin2016penalized},  who used a completely different approach of applying sparse regression to adjust the Gaussian graphical model, though those papers did not offer an explanation.     We conjecture that this phenomenon of maximizing penalized joint likelihood failing to  improve on a  two-stage method may be general.    An intuitive explanation  might lie in the fact that the two parameters $M$ and $\Theta$ are only connected through the penalty:   the Gaussian log-likelihood \eqref{eq:loglike} without a penalty is maximized over $M$ by $\hat M = X$, which does not depend on  $\Theta$. Thus the likelihood itself does not pool information from different observations to estimate the mean (nor should it, since we assumed they are different), while the cohesion penalty is separable in the $p$ variables and does not pool information between them either.   An indirect justification of this conjecture follows from a property of the two-stage estimator stated in Proposition~\ref{thm:oracle} in Appendix~\ref{sec:oracle}, and the numerical results in Section~\ref{sec:sim} provide empirical support.

\subsection{Model selection}

There are two tuning parameters, $\lambda$ and $\alpha$, in the two-stage GNC-lasso algorithm. The parameter $\alpha$ controls the amount of cohesion over the network in the estimated mean and can be easily tuned based on its predictive performance. In subsequent numerical examples, we always choose $\alpha$ from a sequence of candidate values by 10-fold cross-validation. In each fold, the sum of squared prediction errors on the validation set $\sum (X_{ij}-\hat{\mu}_{ij})^2$ is computed and the $\alpha$ value is chosen to minimize the average prediction error. If the problem is too large for cross-validation, we can also use the generalized cross-validation (GCV) statistic as an alternative, which was shown to be effective in theory for ridge-type regularization \citep{golub1979generalized,li1986asymptotic}.  The GCV statistic for $\alpha$ is defined by
$$\text{GCV}(\alpha) = \frac{1}{np}\norm{X-\hat{M}(\alpha)}_F^2/[1-\frac{1}{n}\tr((I+\alpha \lcal_s)^{-1})]^2 = \frac{\norm{X-\hat{M}(\alpha)}_F^2}{np[1-\frac{1}{n}\sum_{i=1}^n\frac{1}{1+\alpha\tau_i}]^2}$$
where we write $\hat{M}(\alpha)$ to emphasize that the estimate depends on $\alpha$.    The parameter $\alpha$ should be selected to minimize GCV.  Empirically, we observe running the true cross-validation is typically more accurate than using GCV. So the GCV is only recommended for problems that are too large to run cross-validation.

Given $\alpha$, we obtain $\hat{M}$ and use $\hat{S} = \frac{1}{n}(X-\hat{M})^T(X-\hat{M})$ as the input of the glasso problem in \eqref{eq:glassoObj}; therefore $\lambda$ can be selected by standard glasso tuning methods, which may depend on the application.  For example, we can tune $\lambda$ according to some data-driven goodness-of-fit criterion such as BIC, or via stability selection.    Alternatively, if the graphical model is being fitted as an exploratory tool to obtain an interpretable dependence between variables, $\lambda$ can be selected to achieve a pre-defined sparsity level of the graph, or chosen subjectively with the goal of interpretability. Tuning illustrates another important advantage of the two-stage estimation over the iterative method: when estimating the parameters jointly, due to the coupling of $\alpha$ and $\lambda$ the tuning must be done on a grid of their values and using the same tuning criteria. The de-coupling of tuning parameters in the two-stage estimation algorithm is both more flexible, since we can use different tuning criteria for each if desired, and  more computationally tractable since we only need to do two line searches instead of a two-dimensional grid search.

\subsection{Related work and alternative penalties}

The Laplacian smoothness penalty of the form \eqref{eq:LaplacianPenalty-scalar} or \eqref{eq:LaplacianPenalty-vec} was originally used in machine learning for embedding and kernel learning \citep{belkin2003laplacian,smola2003kernels, zhou2005learning}.   More recently, this idea has been employed in graph-constrained estimation for variable selection in regression  \citep{li2008network,li2010variable,slawski2010feature,pan2010incorporating,shen2012simultaneous, zhu2013simultaneous, sun2014network, liu2019graph}, principal component analysis \citep{shojaie2010penalized}, and regression inference \citep{zhao2016significance}. In these problems, a network is assumed to connect a set of random variables or predictors and is used to achieve more effective variable selection or dimension reduction in high-dimensional settings.  A generalization to potentially unknown network or group structure was studied by \cite{witten2014cluster}. Though Step 1 of Algorithm~\ref{algo:twostage} has multiple connections to graph constrained estimation, there are a few key differences. In our setting, the network is connecting observations, not variables. We only rely on smoothness across the network for accurate estimation without additional structural assumptions such as sparsity on $M$. In graph-constrained estimation literature, in addition to the Laplacian penalty, other penalties are proposed in special contexts \citep{slawski2010feature,pan2010incorporating,shen2012simultaneous}. We believe similar extensions can also be made in our problem for special applications and we will leave such extensions for future investigation.

An alternative penalty we can impose on $M$ instead of $\sum_{i\sim i'}\norm{\mu_i-\mu_{i'}}^2$ is 
\begin{equation}\label{eq:network-lasso}
\sum_{i\sim i'}\norm{\mu_i-\mu_{i'}}.
\end{equation}
This penalty  is called the network lasso penalty \citep{hallac2015network} and can be viewed as a generalization of the fused lasso \citep{tibshirani2005sparsity} and the group lasso \citep{yuan2006model}. The penalty and its variants were studied recently by \cite{wang2014trenda,jung2018network,tran2018network}.   This penalty is also associated with convex clustering \citep{hocking2011clusterpath,lindsten2011just}, because it typically produces piecewise constant estimates which can be interpreted as clusters.   Properties of convex clustering  have been studied by \cite{hallac2015network,chi2015splitting,tan2015statistical}.  However,  in our setting there are two clear reasons for using the Laplacian penalty and not the network lasso.    First, piecewise constant individual effects within latent clusters of the network is a special case of the general cohesive individual effects, so our assumption is strictly weaker, and there is no reason to impose piecewise constant clusters in the mean unless there is prior knowledge.  Second, solving the optimization in the network lasso problem is computationally challenging and not scalable to the best our knowledge: current state of art algorithms \citep{hallac2015network,chi2015splitting} hardly handle more than about 200 nodes on a single core. In contrast, the Laplacian penalty in \eqref{eq:LaplacianPenalty-vec} admits a closed-form solution and can be efficiently solved for thousands of observations even with a naive implementation on a single machine.  Moreover, there are many ways to improve the naive algorithm based on the special properties of the linear system \citep{spielman2010algorithms,koutis2010approaching, cohen2014solving,sadhanala2016graph,li2016prediction}.   Therefore, \eqref{eq:LaplacianPenalty-vec}  is a better choice than \eqref{eq:network-lasso} for this problem, both computationally and conceptually.

\section{Theoretical properties}\label{sec:theory1}

In this section, we investigate the theoretical properties of the two-stage GNC-lasso estimator. Throughout this section, we assume the observation network $A$ is connected which implies that $\lcal_s$ has exactly one zero eigenvalue. The results can be trivially extended to a network consisting of several connected components, either by assuming the same conditions for each component or regularizing $A$ to be connected as in \cite{amini2013pseudo}. Recall that $\tau_1 \ge \tau_2 \ge  \cdots \ge \tau_{n-1} > \tau_n=0$ are the eigenvalues of $\lcal_s$ corresponding to eigenvectors $u_1, \cdots, u_n$.  For a connected network, we know $\tau_n$ is the only zero eigenvalue. Moreover,$\tau_{n-1}$ is known as \emph{algebraic connectivity} that measure the connectivity of the network. 

\subsection{Cohesion assumptions on the observation network}
The first question we have to address is how to formalize the intuitive notion of cohesion.  We will start with the most intuitive definition of network cohesion for a vector, extend it to a matrix, and then give examples satisfying the cohesion assumptions.  

Intuitively,  we can think of a vector $v \in \bR^n$ as cohesive on a network $A$ if  $v^T\lcal_sv$ is small in some sense, or equivalently, $\norm{\lcal_sv}_2$ is small, since $\lcal_s v$ is the gradient of $v^T\lcal_sv$  up to a constant and  
$$\norm{\lcal_sv}_2 \to 0 \iff v^T\lcal_sv \to 0.$$
It will be convenient to define cohesion in terms of $\lcal_sv$, which also leads to a nice interpretation.   The $i$th coordinate of $\lcal_sv$ can be written as 
$$\frac{d_i}{\bar{d}}  \left( v_i - \frac{1}{d_i}\sum_{i'\sim_{A} i}v_i' \right),$$
which is the difference between the value at node $i$ and the average value of its neighbors, weighed by the degree of node $i$. Let $\lcal_s = U\Lambda U^T$ be the eigen-decomposition of $\lcal_s$, with $\Lambda$ the diagonal matrix with the eigenvalues $\tau_1 \ge \dots \ge \tau_n$ on the diagonal.    The vector $v$  can be expanded in this  basis as $v = U\beta = \sum_{i=1}^n\beta_i v_i$ where $\beta \in \bR^n$. 
Under cohesion, we would expect $\norm{\lcal_sv}_2^2 = \sum_i\tau_i^2\beta_i^2 $ to be much smaller than $\norm{v}_2^2 = \norm{\beta}_2^2 $.     We formalize this in the following definition. 
\begin{defi}[A network-cohesive vector]\label{defi:cohesion}
Given a network $A$ and a vector $v$,  let $v = \sum_{i=1}^n\beta_iu_i$ be the expansion of $v$ in the basis of eigenvectors of $\lcal_s$. We say  $v$ is cohesive on $A$ with rate $\delta>0$ if for all $i = 1, \dots, n$, 
\begin{equation}\label{eq:cohesion}
\frac{\tau_i^2|\beta_i|^2}{\norm{\beta}_2^2} \le  n^{-\frac{2(1+\delta)}{3}-1}, 
\end{equation}
which implies 
$$\frac{\norm{\lcal_sv}_2^2}{\norm{v}_2^2} \le n^{-\frac{2(1+\delta)}{3}}.$$
\end{defi}
Now we can easily define a network-cohesive matrix $M$.
\begin{defi}[A network-cohesive matrix]\label{defi:cohesion2}
A matrix $M \in \bR^{n\times p}$ is cohesive on a network $A$ if all of its columns are cohesive on $A$.
\end{defi}

An obvious but trivial example of a cohesive vector is a constant vector,  which corresponds to $\delta = \infty$.   More generally, we define the class of {\em trivially cohesive} vectors as follows. 
\begin{defi}[A trivially cohesive vector]\label{defi:nontrivialcohesion}
We say vector $v$ is trivially cohesive if
$$\widehat{\var}(v) = o(\bar{v}^2)$$
where $\bar{v} = \sum_{i=1}^n v_i/n$ is the sample mean of $v$, and $\widehat{\var}(v) = \sum_{i=1}^n(v_i-\bar{v})^2/(n-1)$ is the sample variance of $v$.
\end{defi}
Trivial cohesion does not involve a specific network $A$, because such vectors are essentially constant.   We say $v$ is nontrivially cohesive if it is cohesive but not trivially cohesive.   Similarly, we will say a matrix is trivially cohesive if all its columns are trivially cohesive, and nontrivially cohesive if it is cohesive but not trivially cohesive.


For obtaining theoretical guarantees, we will need to make an additional assumption about the network, which essentially quantifies how much network structure can be used to control model complexity under nontrivial cohesion.     This will be quantified through the concept of effective dimension of the network defined below.  
\begin{defi}\label{defi:highrank}
Given a connected network adjacency matrix $A$ of size $n \times n$ and eigenvalues of its standardized Laplacian $\tau_1 \ge \dots \tau_{n-1} > \tau_n = 0$,  define  the \textbf{effective dimension} of the network as 
$$m_A = \inf\{m: 0 \le m \le n-1,  \tau_{n-m} \ge \frac{1}{\sqrt{m}}\}.$$
\end{defi}
Note that spectral graph theory \citep{brouwer2011spectra} implies $\tau_1\ge c$ for some constant $c$, and thus for sufficiently large $n$, we always have $m_A \le n-1$.   
For many sparse and/or structured networks the effective dimension is much smaller than $n-1$, and then we can show nontrivially cohesive vectors/matrices exist.

Our first example of a network with a small effective dimension is a lattice network.    Assume  $\sqrt{n}$ is an integer and define the lattice network of $n$ nodes by arranging them on a $\sqrt{n} \times \sqrt{n}$ spatial grid and connecting grid neighbors (the four corner nodes have degree 2, nodes along the edges of the lattice have degree 3, and all internal nodes have degree 4).  
\begin{prop}[Cohesion on a lattice network]\label{prop:lattice}
Assume $A$ is a lattice network on $n$ nodes, and $\sqrt{n}$ is an integer. Then for a sufficiently large $n$, 
\begin{enumerate}
\item The effective dimension $m_A \le n^{2/3}$.  
\item There exist nontrivially cohesive vectors on the lattice network with rate $\delta=1/2$.
\end{enumerate}
\end{prop}

Figure~\ref{fig:effective-dim} shows the eigenvalues and the function $1/\sqrt{m}$ for reference of a $20\times 20$ lattice and of a coauthorship network we analyze in Section~\ref{sec:app}.   For both networks, the effective dimension is much smaller than the number of nodes:  for the lattice,  $n=400$, while $m_A = 30$  and for the coauthorship network with $n = 635$ nodes, $m_A =  66$.  

\begin{figure}[H]
\centering
\begin{subfigure}{.5\textwidth}
  \centering
  \includegraphics[width=\linewidth]{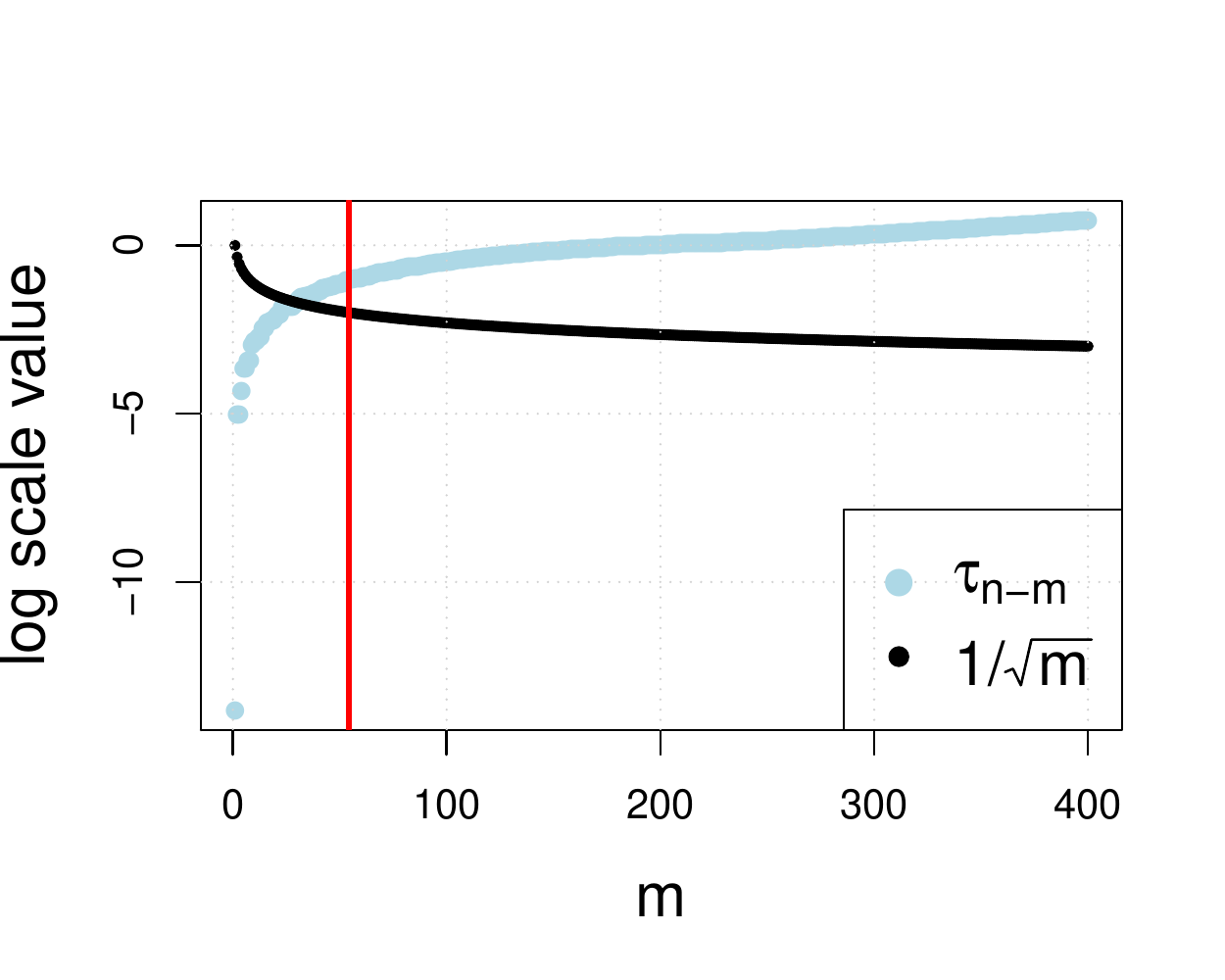}
  \caption{$20\times 20$ lattice}
\end{subfigure}%
\begin{subfigure}{.5\textwidth}
  \centering
  \includegraphics[width=\linewidth]{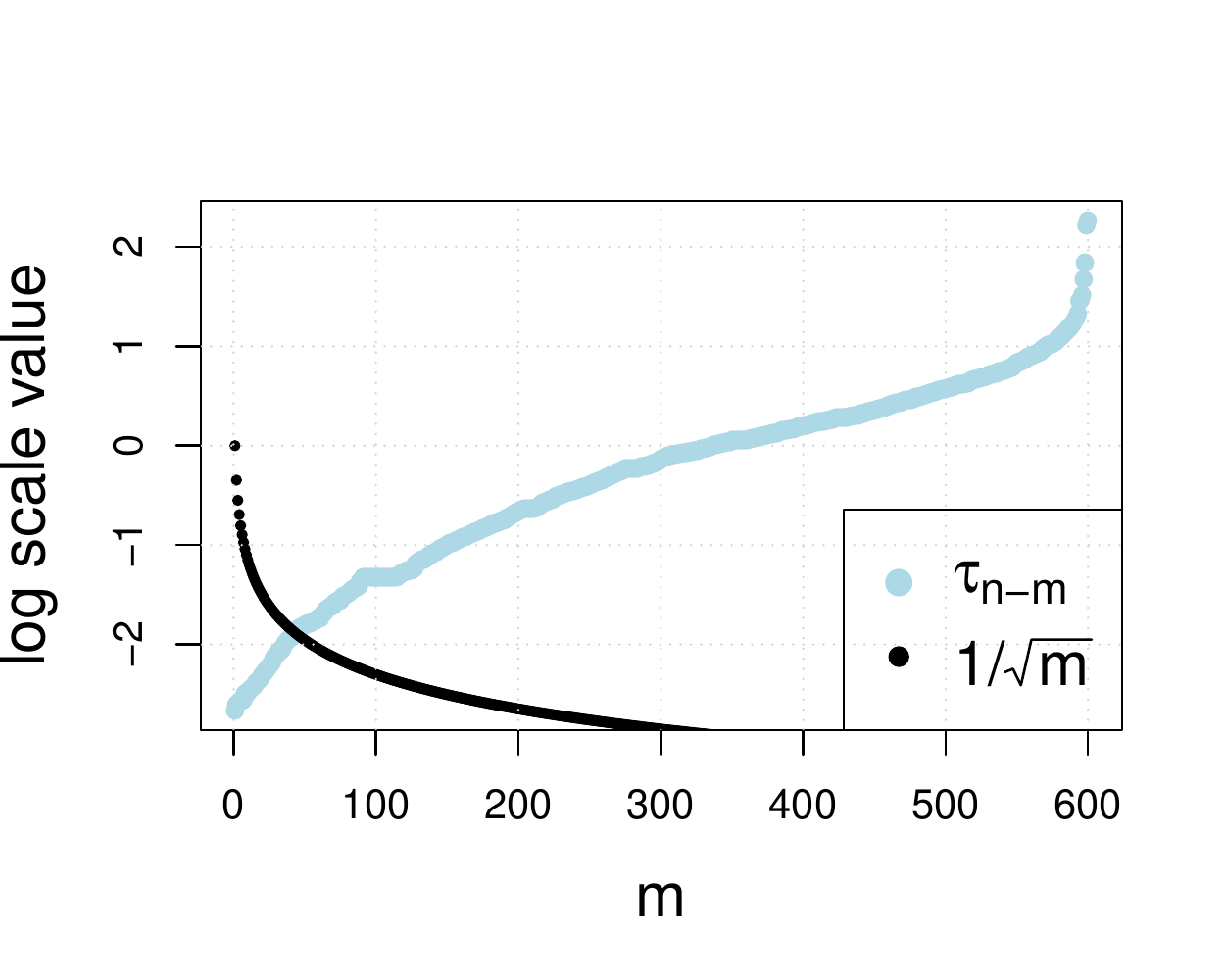}
  \caption{Coauthorship network}
\end{subfigure}
\caption{Eigenvalues and effective dimensions of a $20\times 20$ lattice and the coauthorship networks from Section~\ref{sec:app}. The red vertical line in the left panel is  $n^{2/3}$, the theoretical upper bound from  Proposition~\ref{prop:lattice}.}
\label{fig:effective-dim}
\end{figure}

\subsection{Mean estimation error bound}

Our goal here is to obtain a bound on the difference between $M$ and the estimated $\hat{M}$ obtained by Algorithm~\ref{algo:twostage}, under the following cohesion assumption.
\begin{ass}\label{ass:regressionCoef}
The mean matrix $M$ is cohesive over the network $A$ with rate $\delta$ where $\delta$ is a positive constant. Moreover, $\norm{M_{\cdot j}}_2^2 \le b^2 n$ for every $j \in [p]$ for some positive constant $b$.
\end{ass}

\begin{thm}[Mean error bound]\label{thm:MeanInitialError}
  Assume model \eqref{eq:GNC-GGM} and  Assumption~\ref{ass:regressionCoef} are true.  Write $\sigma^2 = \max_j \Sigma_{jj}$ and $\Delta = n^{\frac{1+\delta}{3}}\tau_{n-1}$, where $\tau_{n-1}$ is the smallest nonzero eigenvalue of $\lcal_s$.    Then $\hat{M}$ estimated by  \eqref{eq:individualRNC} with $\alpha = n^{\frac{1+\delta}{3}}$ satisfies 
\begin{equation}\label{eq:MFrobeniusBound}
\frac{\norm{\hat{M}-M}_{F}^2}{np} \le  \frac{(b^2+2\sigma^2)[1+m_A(\frac{1}{(1+\Delta)^2}+n^{\frac{1-2\delta}{3}})]}{n}
\end{equation}
with probability at least $1-\exp(-c(n-m_A)r(\Sigma)) - \exp(-cm_Ar(\Sigma))- \exp(-c\frac{p\sigma^2}{\phi_{max}(\Sigma)})$ for some positive constant $c$, where $m_A$ is the effective dimension of network $A$ in Definition~\ref{defi:highrank} and $r(\Sigma)$ is the stable rank of $\Sigma$.
\end{thm}

The theorem shows that the average estimation error is vanishing with high probability as long as the cohesive dimension $m_A = o(n^{\frac{2(1+\delta)}{3}})$ while $m_Ar(\Sigma)$ and $p/\phi_{max}(\Sigma)$ grow with $n$.    Except for degenerate situations, we would expect $r(\Sigma)$ and $p/\phi_{max}(\Sigma)$ to grow with $p$,  which in turn grows with $n$. In \eqref{eq:MFrobeniusBound}, the term $\Delta = n^{(1+\delta)/3}\tau_{n-1}$ involves both the cohesion rate of the mean matrix and the algebraic connectivity of the network. In trivially cohesive settings, $\delta \to \infty$ and $\Delta \to \infty$ so the bound does not depend on the network, and the error bound becomes the standard mean estimation error bound.    General lower bounds for $\tau_{n-1}$ are available \citep{fiedler1973algebraic}, but we prefer not to introduce additional algebraic definitions at this point.

  Finally, note that the value of $\alpha$ depends on the cohesive rate $\delta$ of $M$. Therefore, the theorem is not adaptive to the unknown cohesive rate. In practice, as we discussed, one has to use cross-validation to tune $\delta$.

\subsection{Inverse covariance estimation error bounds}\label{secsec:cov}
Our next step is to show that $\hat{M}$ is a sufficiently accurate estimate of $M$ to guarantee good properties of the estimated precision matrix $\Theta$ in step 2 of the  two-stage GNC-lasso algorithm.   
We will need some additional assumptions, the same ones needed for the glasso performance guarantees under the standard Gaussian graphical model \citep{rothman2008sparse, ravikumar2011high}.


Let $\Gamma = \Sigma\otimes\Sigma$ be the Fisher information matrix of the model,  where $\otimes$ denotes the Kronecker product. In particular, under the multivariate Gaussian distribution, we have $\Gamma_{(j,k),(\ell,m)} = \cov(X_jX_k,X_{\ell}X_m)$. Define the set of nonzero entries of $\Theta$ as 
\begin{equation}\label{eq:support}
S(\Theta) = \{(j,j') \in [n]\times [n]: \Theta_{jj'}\ne 0\}.
\end{equation}

We use $S^c(\Theta)$ to denote the complement of $S(\Theta)$.   Let $s=|S(\Theta)|$ be the number of nonzero elements in $\Theta$. Recall that we assume all diagonals of $\Theta$ are nonzero. For any two sets $T_1, T_2 \subset [n]$, let $\Gamma_{T_1, T_2}$ denote the submatrix with rows and columns indexed by $T_1$, $T_2$, respectively. When  the context is clear, we may simply write $S$ for $S(\Theta)$.    Define 
\begin{eqnarray*}
\psi & = & \max_j\norm{\Theta_{j\cdot}}_0 , \\
\kappa_{\Sigma} & = & \norm{\Sigma}_{\infty,\infty} , \\
\kappa_{\Gamma} & = & \norm{(\Gamma_{SS})^{-1}}_{\infty,\infty}
  \end{eqnarray*}
where the vector operator $\norm{\cdot }_0$ gives the number of nonzeros in the vector while the matrix norm $\norm{\cdot} _{\infty,\infty}$ gives the maximum $L_{\infty}$ norm of the rows. 

Finally, by analogy to the well-known irrepresentability condition for the lasso, which is necessary and sufficient for the lasso to recover support  \citep{wainwright2009sharp}, we need an edge-level irrepresentability condition.
\begin{ass}\label{ass:irrepresentability}
There exists some $0 < \rho \le1$ such that
\begin{equation*}
\max_{e\in S^c}\norm{\Gamma_{eS}(\Gamma_{SS})^{-1}}_1 \le 1-\rho.
\end{equation*}
\end{ass}

If we only want to obtain a Frobenius norm error bound, the following much weaker assumption  is sufficient,  without conditions on $\psi, \kappa_{\Sigma}, \kappa_{\Gamma}$ and Assumption~\ref{ass:irrepresentability}:
\begin{ass}\label{ass:spectralBound}
Let $\eta_{\min}(\Sigma)$ and $\eta_{\max}(\Sigma)$ be the minimum and maximum eigenvalues of $\Sigma$, respectively. There exists a constant $\bar{k}$ such that
\begin{equation*}
\frac{1}{\bar{k}} \le \eta_{\min}(\Sigma) \le \eta_{\max}(\Sigma) \le \bar{k}.
\end{equation*}
\end{ass}

Let $\hat{S} = \frac{1}{n}(X-\hat{M})^T(X-\hat{M}).$  We use $\hat{S}$ as input for the glasso estimator \eqref{eq:glassoObj}. We would expect that if $\hat{M}$ is an accurate estimate of $M$, then $\Theta$ can be accurately estimated by glasso.  The following theorem formalizes this intuition, using concentration properties of $\hat{S}$ around $\Sigma$ and the proof strategy of \cite{ravikumar2011high}.   We present the high-dimensional regime result here, with $p \ge n^{c_0}$ for some positive constant $c_0$, and state the more general  result which includes the lower-dimensional regime in Theorem~\ref{thm:two-stageMaxBound'} in the  Appendix, because the general form is more involved.

\begin{thm}\label{thm:two-stageMaxBound}
Under the conditions of Theorem~\ref{thm:MeanInitialError} and Assumption~\ref{ass:irrepresentability}, suppose there exists some positive constant $c_0$ such that $p \ge n^{c_0}$. If  $\log{p} = o(n)$ and $m_A= o(n)$, there exist some positive constants $C,c, c', c''$ that only depend on $c_0, b$ and $\sigma$, such that if $\hat{\Theta}$ is the output of Algorithm~\ref{algo:twostage} with $\alpha = n^{\frac{1+\delta}{3}}$, $\lambda = \frac{8}{\rho}\nu(n,p)$ where
\begin{align}\label{eq:nu}
\nu(n,p) :=C\sqrt{\frac{\log p}{n}}\max\Big(1, m_A n^{-\frac{1+4\delta}{6}}, \sqrt{m_A}n^{\frac{1-2\delta}{6}} , \sqrt{\frac{\log{p}}{n} }(\frac{m_A}{\Delta+1}+1)(\sqrt{m_A}n^{-\frac{1+\delta}{3}} +   1 ) \Big)
\end{align}
and $n$ sufficiently large so that  
 $$\nu(n,p)<\frac{1}{6(1+8/\rho)\psi\max\{\kappa_{\Sigma}\kappa_{\Gamma},  (1+8/\rho)\kappa_{\Sigma}^3\kappa_{\Gamma}^2   \}},$$
 then with probability at least $1-\exp(-c\log(p(n-m_A))) - \exp(-c'\log(pm_A)) - \exp(-c''\log{p})$,  then the estimate $\hat{\Theta}$ has the following properties: 
\begin{enumerate}
\item  Error bounds:  
\begin{eqnarray*}
\norm{\hat{\Theta}-\Theta}_{\infty} & \le &  2(1+8/\rho)\kappa_{\Gamma}\nu(n,p)
  \\
   \norm{\hat{\Theta}-\Theta}_F & \le  & 2(1+8/\rho)\kappa_{\Gamma}\nu(n,p)\sqrt{s+p}.  \label{eq:FrobeniusErrorBound}\\
   \norm{\hat{\Theta}-\Theta}~~ &\le& 2(1+8/\rho)\kappa_{\Gamma}\nu(n,p)\min(\sqrt{s+p}, \psi).    
 \end{eqnarray*}

 \item Support recovery: 
   $$S(\hat{\Theta}) \subset S(\Theta), $$
   and if additionally 
$\min_{(j,j')\in S(\Theta)}|\Theta_{jj'}| > 2(1+8/\rho)\kappa_{\Gamma}\nu(n,p),$
then $$S(\hat{\Theta}) = S(\Theta).$$

%
\end{enumerate}
\end{thm}

\begin{rem}
As commonly assumed in literature, such as \cite{ravikumar2011high}, we will treat $\kappa_{\Gamma}$, $\kappa_{\Sigma}$ and $\rho$ to be constants or bounded.
\end{rem}
\begin{rem}
 The Frobenius norm bound does not need the strong irrepresentability assumption and does not depend on $\kappa_{\Gamma}$ and $\kappa_{\Sigma}$.  Following the proof strategy in \cite{rothman2008sparse}, this bound can be obtained under the much weaker Assumption~\ref{ass:spectralBound} instead. 
\end{rem}

The quantity in \eqref{eq:nu} involves four terms. The first term is from the inverse covariance estimation with a known $M$ (a standard glasso problem), and the other three terms come from having to estimate a cohesive $M$.   These three terms depend on both the cohesion rate and the effective dimension of the network. As expected, they all increase with $m_A$ and decrease with $\delta$. 
The last term also involves $\Delta$, which depends on both $\delta$ and the algebraic connectivity $\tau_{n-1}$.     To illustrate these trade-offs, we consider the implications of Theorem~\ref{thm:two-stageMaxBound} in a few special settings.

First, consider the setting of trivial cohesion, with $\delta = \infty$. In this case, the last three terms in \eqref{eq:nu} vanish.
\begin{coro}\label{coro:trivial}
Under the assumptions of Theorem~\ref{thm:two-stageMaxBound}, if $M$ is trivially cohesive and $\delta = \infty$, then all results of Theorem~\ref{thm:two-stageMaxBound} hold with 
$$\nu(n,p) = C\sqrt{\frac{\log p}{n}} , $$ 
and the estimated $\hat{\Theta}$ is consistent as long as $\log{p} = o(n).$
\end{coro}
This result coincides with the standard glasso error bound from \cite{ravikumar2011high}.   Thus when $M$ does not vary, we do not lose anything in the rate by using GNC-lasso instead of glasso.

Another illustrative setting is the case of bounded effective dimension $m_A$.   Then the third term in \eqref{eq:nu} dominates.  
\begin{coro}\label{coro:trivial}
Under the assumptions of Theorem~\ref{thm:two-stageMaxBound}, if the network has a bounded effective dimension $m_A$, then all the results of Theorem~\ref{thm:two-stageMaxBound} hold with 
$$\nu(n,p) = C\sqrt{\frac{\log p}{n}}n^{\frac{\max(1-2\delta,0)}{6}}.$$ 
In particular, if $\delta \ge 1/2$, $\hat{\Theta}$ is consistent as long as $\log{p} = o(n)$.
\end{coro}
This corollary indicates that if the network structure is favorable to cohesion, the GNC-lasso does not sacrifice anything in the rate up to a certain level of nontrivial cohesion.

Finally, consider a less favorable example in which $\log{p} = o(n)$ may no longer be  enough for consistency.    Recall Proposition~\ref{prop:lattice} indicates $m_A  = O(n^{2/3})$ for lattice networks, and suppose the cohesive can be highly nontrivial.
\begin{coro}[Consistency on a $\sqrt{n}\times \sqrt{n}$ lattice]\label{coro:lattice}
Suppose the conditions of Theorem~\ref{thm:two-stageMaxBound} hold and $m_A\le n^{2/3}$.  The GNC-lasso estimate $\hat \Theta$ is consistent  if  $\delta > 3/8$ and
$$\log{p} = o(n^{\min(1, 8\delta-3 )/3}).$$
In particular, if $\delta = 1/2$, it is necessary to have $\log p = o(n^{1/3})$ for consistency.
\end{coro}
The corollary suggests that consistency under some regimes of nontrivial cohesion requires strictly stronger conditions than $\log p = o(n)$.   Moreover, if cohesion is too weak (say, $\delta \le 3/8$), consistency cannot be guaranteed by these results.

\section{Simulation studies}\label{sec:sim}
We evaluate the new GNC-lasso method and compare it to some baseline alternative methods in simulations based on both synthetic and real networks.   The synthetic network we use is a $20\times 20$ lattice network with $n = 400$ nodes and  a vector with dimension $p = 500$ observed at each node; this setting satisfies the assumptions made in our theoretical analysis.    We also test our method on the coauthorship network shown in Figure~\ref{fig:CoauthorNet}, which will be described in Section~\ref{sec:app}. This network has $n=635$ nodes at $p=800$ observed features at each node.

\paragraph{Noise settings:} The conditional dependence graph $\gcal$ in the Gaussian graphical model is generated as an Erd\"{o}s-Renyi graph on $p$ nodes,  with each node pair connecting independently with probability 0.01. The Gaussian noise is then drawn  from $\ncal(0, \Sigma)$ where $\Theta = \Sigma^{-1} =  a(0.3A_{\gcal} + (0.3e_{\gcal} + 0.1)  I)$,  where $A_{\gcal}$ is the adjacency matrix of $\gcal$,$e_{\gcal}$ is the absolute value of the smallest eigenvalue of $A_{\gcal}$ and the scalar $a$ is set to ensure the resulting $\Sigma$ has all diagonal elements equal to 1. This procedure is implemented in \cite{zhao2012huge}.

\paragraph{Mean settings:}    We set up the mean to allow for varying degrees of cohesion.  each row $M_{\cdot j}, j=1,2, \cdots, p$ as
\begin{equation}\label{eq:M-gen}
M_{\cdot, j}  = \sqrt{t}\sqrt{n}u^{(j)} + \sqrt{1-t} \mbone
\end{equation}
where $u^{(j)}$ is randomly sampled with replacement from the eigenvectors of the Laplacian $u_{n-1}, n_{n-2}, \cdots, u_{n-k}$ for some integer $k$ and $t$ is the mixing proportion.  We then rescale $M$ so the signal-to-noise ratio becomes 1.6, so that the problem remains solvable to good accuracy by proper methods but is not too easy to solve by naive methods.   In a connected network, the constant vector is trivially cohesive. The cohesion becomes increasingly nontrivial as one increases $k$ and $t$. For example, $t=0$ gives identical mean vectors for all observations and as $t$ increases, the means become more different.  The integer $k$ is chosen to give a reasonably eigen-gap in eigenvalues, with details in subsequent paragraphs.

We evaluate performance on recovering the true underlying graph by the receiver operating characteristic (ROC) curve, along a graph estimation path obtained by varying $\lambda$. An ROC curve illustrates the tradeoff between the true positive rate  ($\text{TPR}$) and the false positive rate ($\text{FPR}$), defined as 
\begin{align*}
 \text{TPR} &  = \frac{\#\{(j,j'):j\ne j', \Theta_{jj'}\ne 0, \hat{\Theta}_{jj'}\ne 0\}}{\#\{(j,j'):j\ne j', \Theta_{jj'}\ne 0\}} \\
  \text{FPR} & = \frac{\#\{(j,j'):j\ne j', \Theta_{jj'}= 0, \hat{\Theta}_{jj'}\ne 0\}}{\#\{(j,j'):j\ne j', \Theta_{jj'}= 0\}}.\end{align*}
We also evaluate the methods on the estimation error of $M$, measured as $\norm{\hat{M}-M}_{\infty} = \max_{ij}|\hat{M}_{ij}-M_{ij}|$ for the worst-case entry-wise recovery and $\norm{\hat{M}-M}_{2,\infty} = \max_{i}\norm{\hat{M}_{i\cdot}-M_{i\cdot}}$ for the worst-case mean vector error for each observation.

As a baseline comparison, we include the standard glasso which does not use the network information at all.  We also compare to a natural approach to incorporating heterogeneity without using the network;  we do this by applying $K$-means clustering to group observations into clusters, estimating a common mean for each cluster, and applying glasso after centering each group with its own mean.   This approach requires estimating the number of clusters. However, the widely used gap method \citep{tibshirani2001estimating} always suggests only one cluster in our experiments, which defaults back to glasso.   Instead, we picked the number of clusters to give the  highest area under the ROC curve; we call this result ``oracle cluster+glasso" to emphasize that it will not be feasible in practice.    For GNC-lasso, we report both the oracle tuning (the highest AUC, not available in practice) and 10-fold cross-validation based tuning, which we recommend in practice.  The oracle methods serve as benchmarks for the best possible performance available from each method.

\subsection{Performance as a function of cohesion}\label{secsec:cohesion-eval}

\begin{figure}[H]
\centering
\begin{subfigure}{.33\textwidth}
  \centering
  \includegraphics[width=\linewidth]{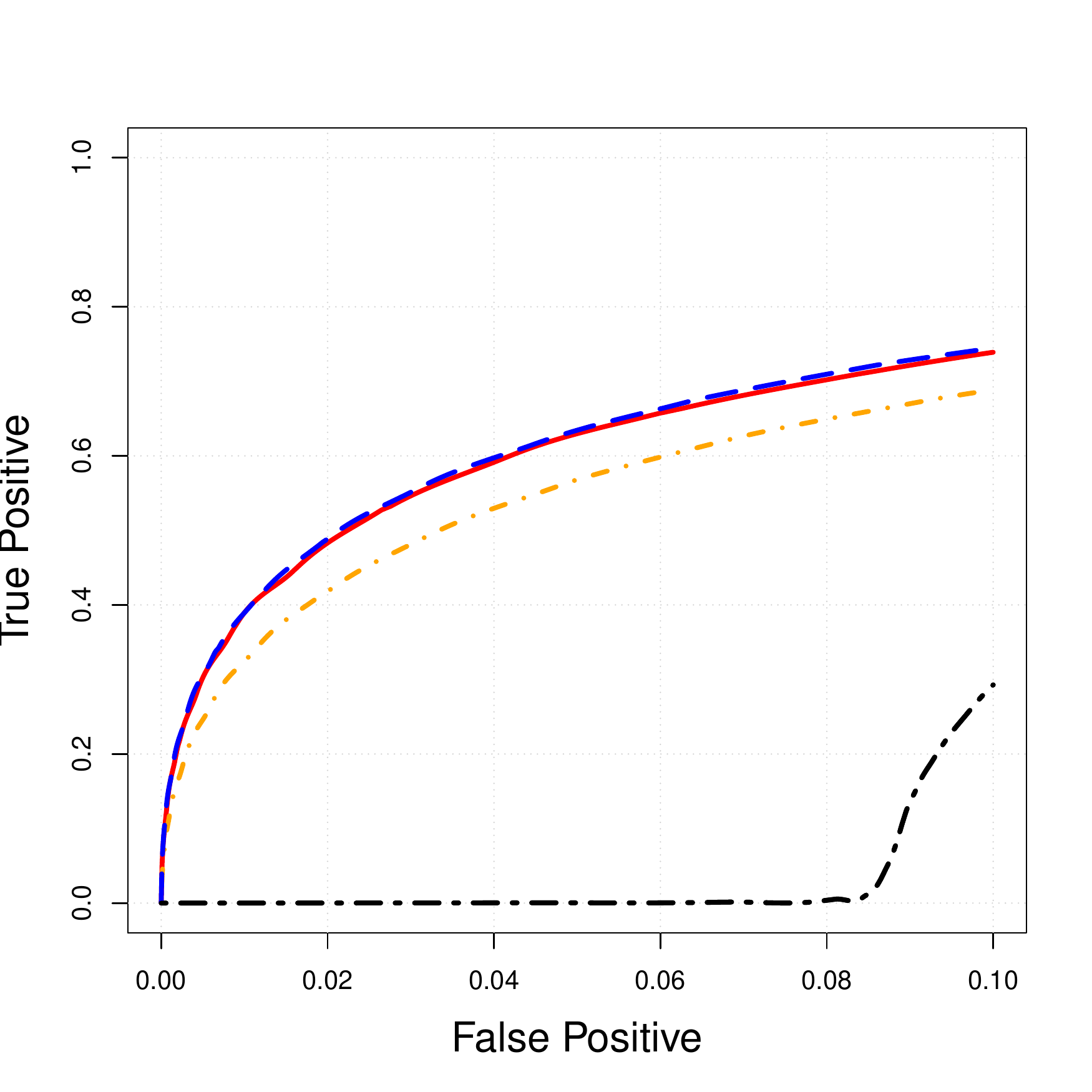}
  \caption{$t=0.1$}
\end{subfigure}%
\begin{subfigure}{.33\textwidth}
  \centering
  \includegraphics[width=\linewidth]{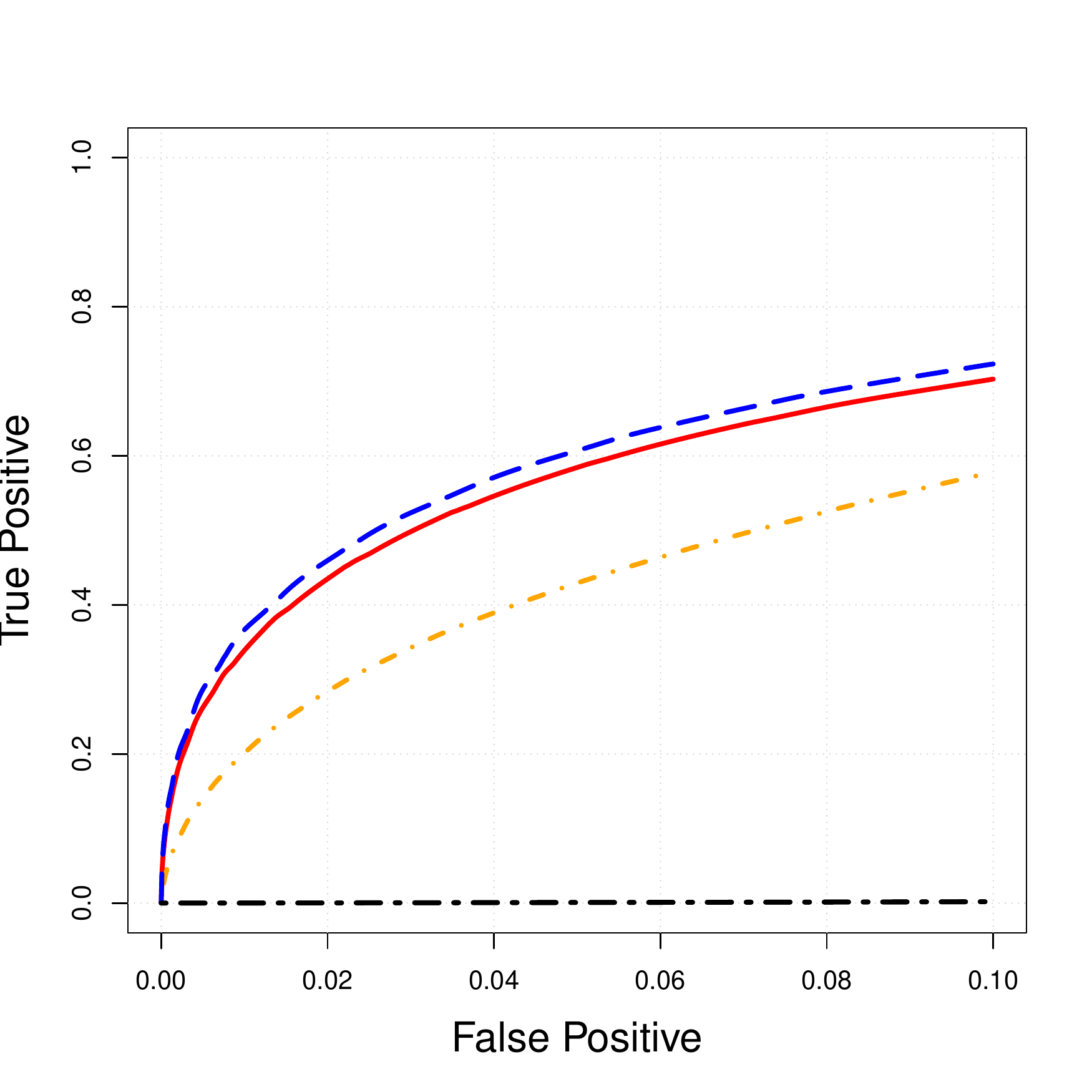}
  \caption{$t=0.5$}
\end{subfigure}
\begin{subfigure}{.33\textwidth}
  \centering
  \includegraphics[width=\linewidth]{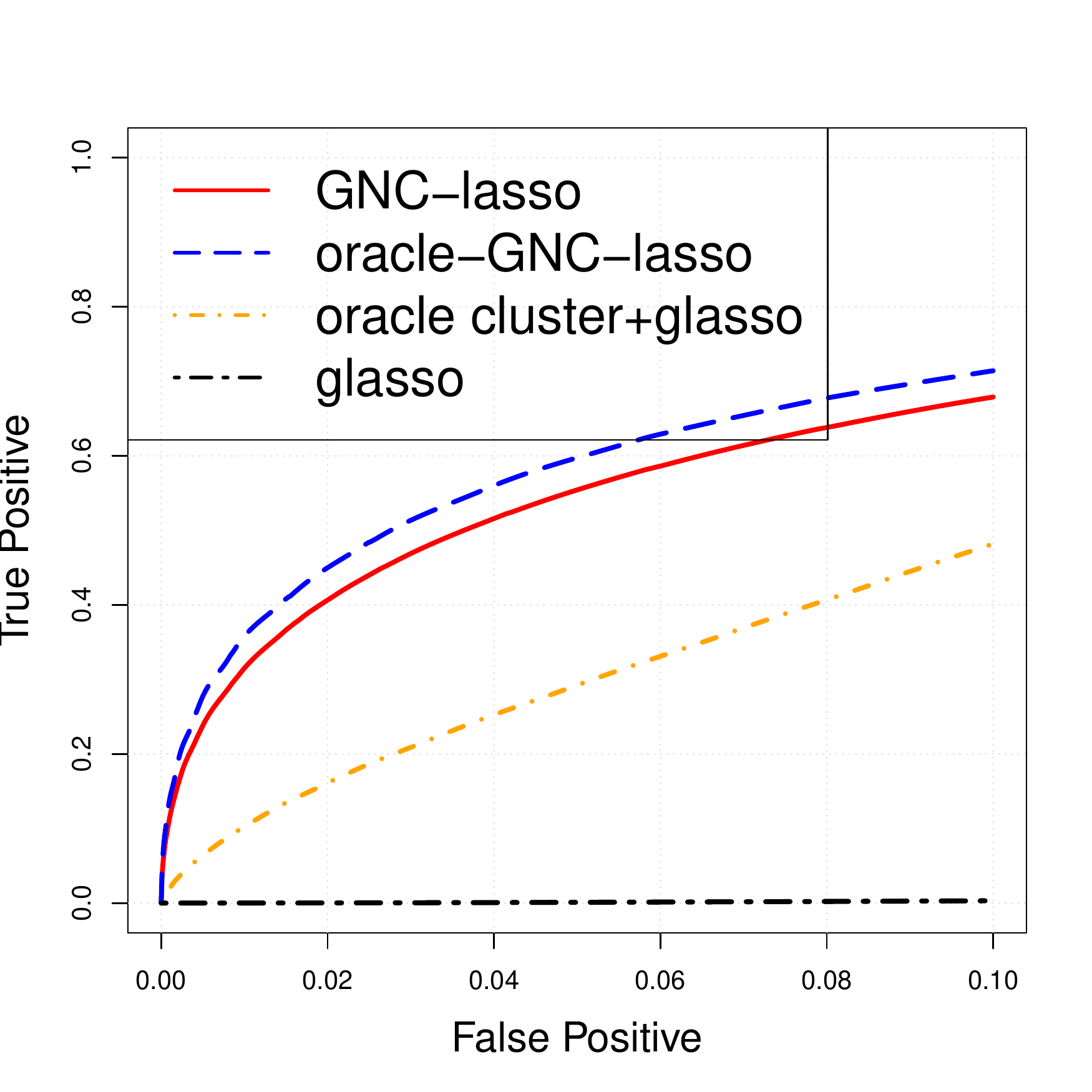}
  \caption{$t=1$}
\end{subfigure}

\caption{Graph recovery ROC curves under three different levels of cohesion corresponding to  $t=0.1,  0.5, 1$, for the lattice network ($n=400$, $p=500$). }
\label{fig:sim1}
\end{figure}

First, we vary the level of cohesion in the mean, by setting $t$ to 0.1, 0.5, or 1, corresponding to strong, moderate, or weak cohesion.     Figure~\ref{fig:sim1} shows the ROC curves of the four methods obtained from 100 independent replications for the lattice network.  Glasso fails completely even when the model has only a slight amount of heterogeneity ($t=0.1$).    Numerically, we also observed that heterogeneity slows down convergence for glasso.   The oracle cluster+glasso improves on glasso as it can accommodate some heterogeneity, but is not comparable to GNC-lasso.  As $t$ increases, the GNC-lasso maintains similar levels of performance by adapting to varying heterogeneity, while the oracle cluster+glasso degrades quickly, since for more heterogeneous means the network provides much more reliable information than $K$-means clustering on the observations.   We also observed that cross-validation is similar to oracle tuning for GNC-lasso, giving it another advantage.    Figure~\ref{fig:sim2} shows the results for the same setting but on the real coauthorship network instead of the lattice.    The results are very similar to what we obtained on the lattice, giving further support to GNC-lasso practical relevance.

We also compare  estimation errors in $\hat{M}$  in Table~\ref{tab:M-estimate}.  The oracle GNC-lasso is almost always the best, except for one setting where it is inferior to the CV-tuned GNC-lasso (note that the ``oracle" is defined by the  AUC and is thus not guaranteed to produce the lowest error in estimating $M$).    For the lattice network, cluster + glasso does comparably to GNC-lasso (sometimes better, and sometimes worse).  For the coauthorship network, a more realistic setting, GNC-lasso is always comparable to the oracle and substantially better than both alternatives that do not use the network information.

\begin{figure}[H]
\centering
\begin{subfigure}{.33\textwidth}
  \centering
  \includegraphics[width=\linewidth]{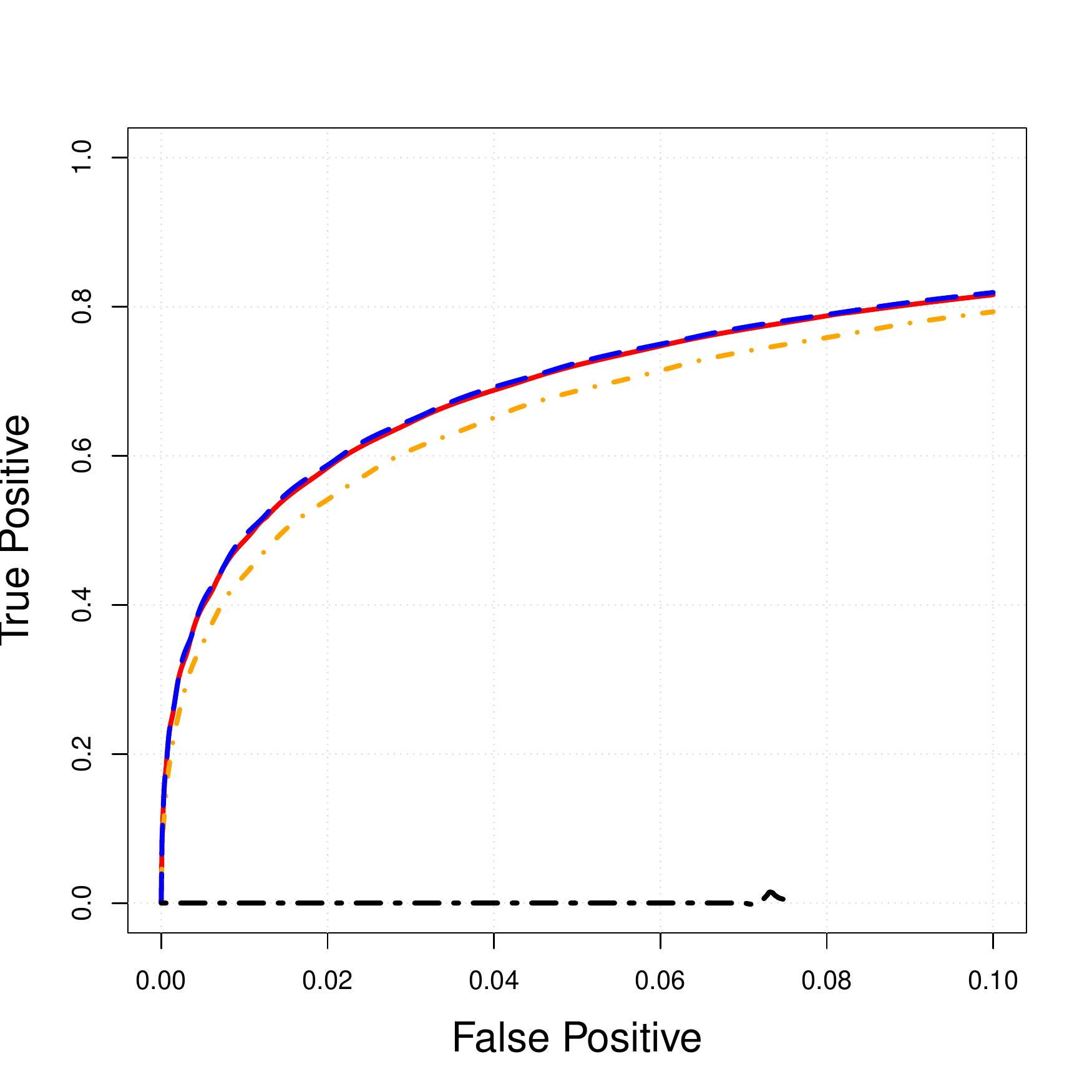}
  \caption{$t=0.1$}
\end{subfigure}%
\begin{subfigure}{.33\textwidth}
  \centering
  \includegraphics[width=\linewidth]{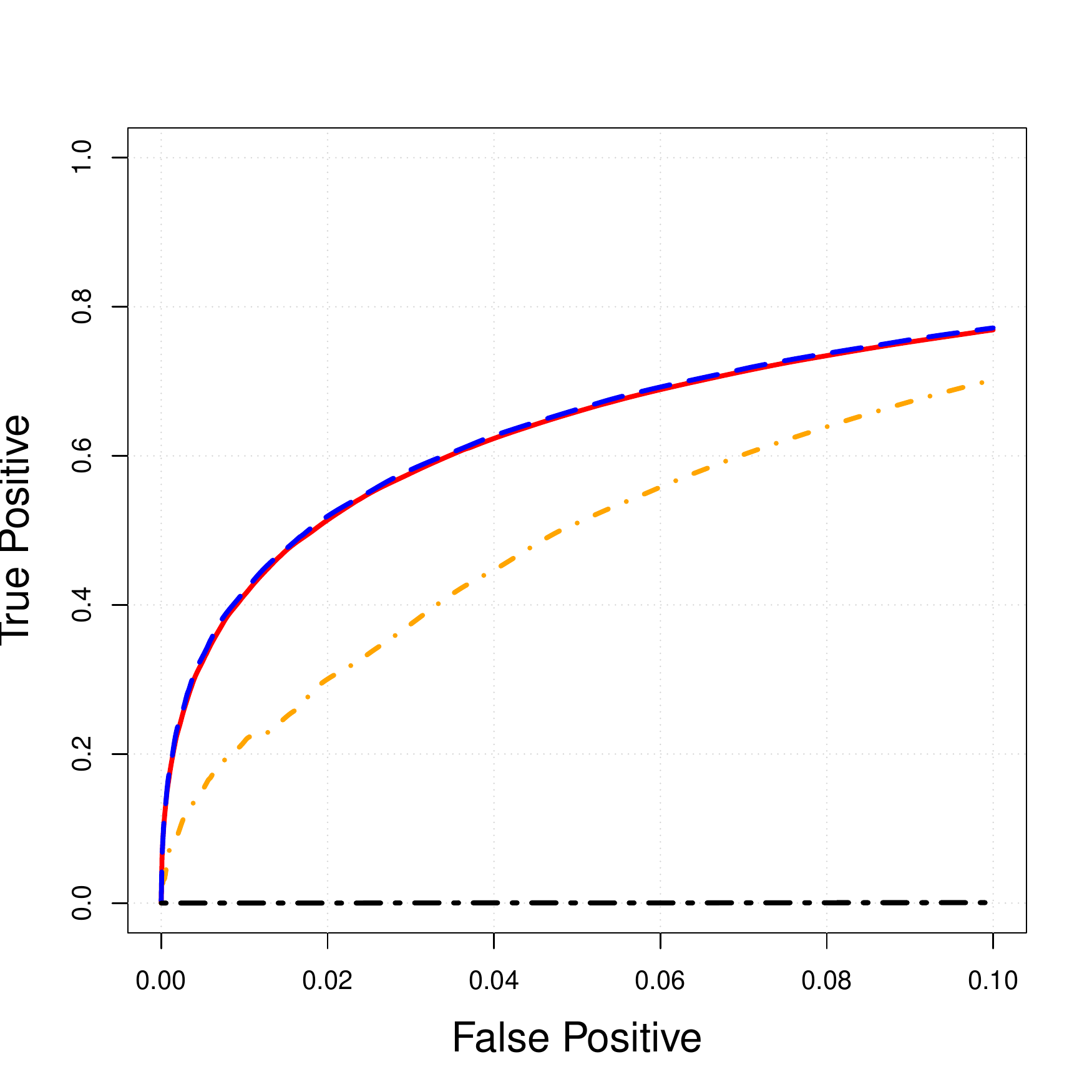}
  \caption{$t=0.5$}
\end{subfigure}
\begin{subfigure}{.33\textwidth}
  \centering
  \includegraphics[width=\linewidth]{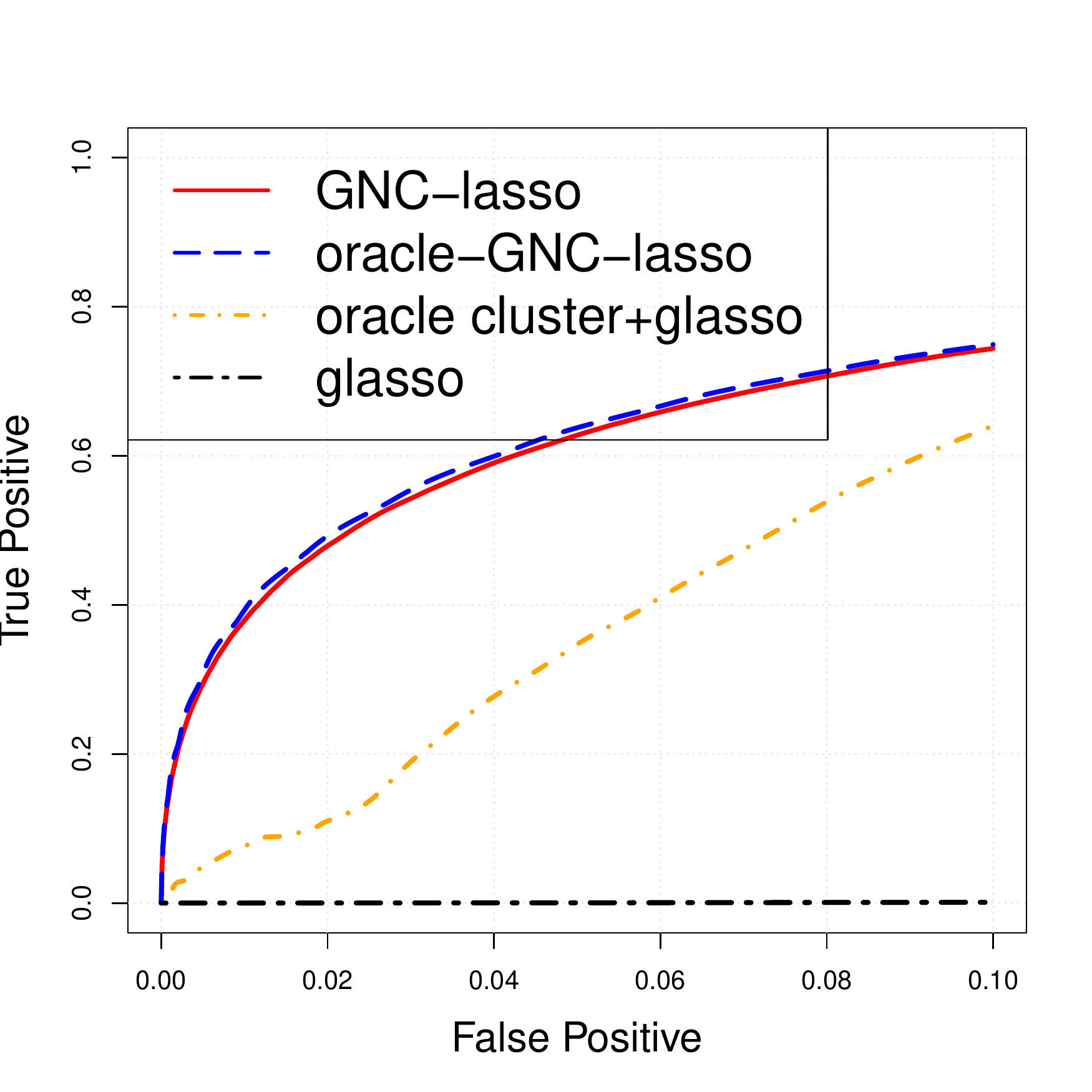}
  \caption{$t=1$}
\end{subfigure}

\caption{Graph recovery ROC curves under three different levels of cohesion corresponding to  $t=0.1,  0.5, 1$, for the coauthorship network ($n=635$, $p=800$).}
\label{fig:sim2}
\end{figure}

\begin{table}[ht]
\centering
\caption{Mean estimation errors for the four methods, averaged over 100 replications, with the lowest error in each configuration indicated in bold. }
\label{tab:M-estimate}
\begin{tabular}{r|r|rrr|rrr}
  \hline
  & & \multicolumn{3}{c}{$\norm{\hat{M}-M}_{\infty}$}& \multicolumn{3}{c}{$\norm{\hat{M}-M}_{2,\infty}$}\\
  \hline
 network & method & $t=0.1$ & 0.5 & 1 & $t=0.1$ & 0.5 & 1 \\ 
  \hline
\multirow{ 4}{*}{lattice} & glasso& 0.358 & 0.746 & 1.037 & 5.819 & 12.985 & 18.357 \\ 
&oracle cluster+glasso & 0.493 & 0.539 & 0.565 & 3.293 & 3.639 & 4.118 \\ 
&oracle GNC-lasso & {\bf 0.328} & {\bf 0.520} & {\bf 0.526} & {\bf 2.054} & {\bf 3.247} & {\bf 3.287} \\ 
& GNC-lasso & 0.419 & 0.669 & 0.820 & 2.619 & 4.105 & 4.874 \\ 
   \hline
\multirow{ 4}{*}{coauthorship} & glasso & 1.540 & 3.401 & 4.795 & 25.072 & 57.655 & 78.426 \\ 
&oracle cluster+glasso & 0.724 & 1.077 & 1.342 & 7.051 & 13.078 & 16.962 \\ 
&oracle GNC-lasso & {\bf 0.710} & {\bf 0.860} & {\bf 0.917} & {\bf 6.400} & 7.404 & {\bf 7.420} \\ 
&GNC-lasso & 0.717 & 0.878 & 0.942 & 6.430 & {\bf 7.037} & 7.436 \\ 
   \hline
\end{tabular}
\end{table}

\subsection{Performance as a function of sparsity}

A potential challenge for GNC-lasso is a sparse network that does not provide much information, and in particular a network with multiple connected components.   
As a simple test of what happens when a network has multiple components, we split the $20\times 20$ lattice  into either four disconnected $10\times 10$ lattice subnetworks, or 16 disconnected $5\times 5$ lattice subnetworks, by removing all edges between these subnetworks.  The data size ($n=400, p=500$) and the data generating mechanism remain the same; we set $t = 0.5$ for a moderate degree of cohesion. The only difference here is when there are $K$ connected components in the network, the last $K$ eigenvectors of the Laplacian $u_n, \cdots, u_{n-K+1}$ are all constant within each connected component (and thus trivially cohesive). Therefore, in the case of 4 disconnected subnetworks, we randomly sample the last $k=12$ eigenvectors to generate $M$ in \eqref{eq:M-gen} while in the case of 16 disconnected subnetworks, we set $k=48$.   The effective dimensions $m_A$ are 30, 32, and 48, respectively.

Similarly, we also split the coauthorship network into two or four subnetworks by applying hierarchical clustering in \cite{li2018hierarchical}, which is designed to separate high-level network communities (if they exist).    We then remove all edges between the communities found by clustering to produce  a network with either two or four connected components. To generate $M$ from \eqref{eq:M-gen}, we use $k=6$ for two components and $k=12$ for four components, and again set $t = 0.5$ for moderate cohesion. The effective dimension  $m_A$ becomes 66, 74, and 78,  respectively.

Figure~\ref{fig:Lattice-Disconnected} shows the ROC curves and Table~\ref{tab:M-estimate-disconnected-grid} shows the mean estimation errors  for the three versions of the lattice network. Overall, all methods get worse as the network is split, but the drop in performance is fairly small for the oracle GNC-lasso.   Cross-validated GNC lasso suffers slightly more from splitting (the connected components in the last case only have 25 nodes each, which can produce isolated nodes and hurt cross-validation performance).   Again,  both GNS methods are much more accurate than the two benchmarks (glasso completely fails, and oracle cluster+glasso performance substantially worse). 

  Figure~\ref{fig:Coauthor-Disconnected} and Table~\ref{tab:M-estimate-disconnected-coauthor} give the results for the three versions of the coauthorship network.     The network remains well connected in all configurations and both the oracle and the cross-validated GNC-lasso perform well in all three cases, without deterioration.   The oracle cluster+glasso performs well in this case as well, but GNC-lasso still does better on both graph recovery and estimating the mean.  Glasso fails completely once again.

\begin{figure}[H]
\centering
\begin{subfigure}{.33\textwidth}
  \centering
  \includegraphics[width=\linewidth]{./Figures/R1-GridNet-Mix050-SNC060}
  \caption{Original $20\times 20$ lattice.}
\end{subfigure}%
\begin{subfigure}{.33\textwidth}
  \centering
  \includegraphics[width=\linewidth]{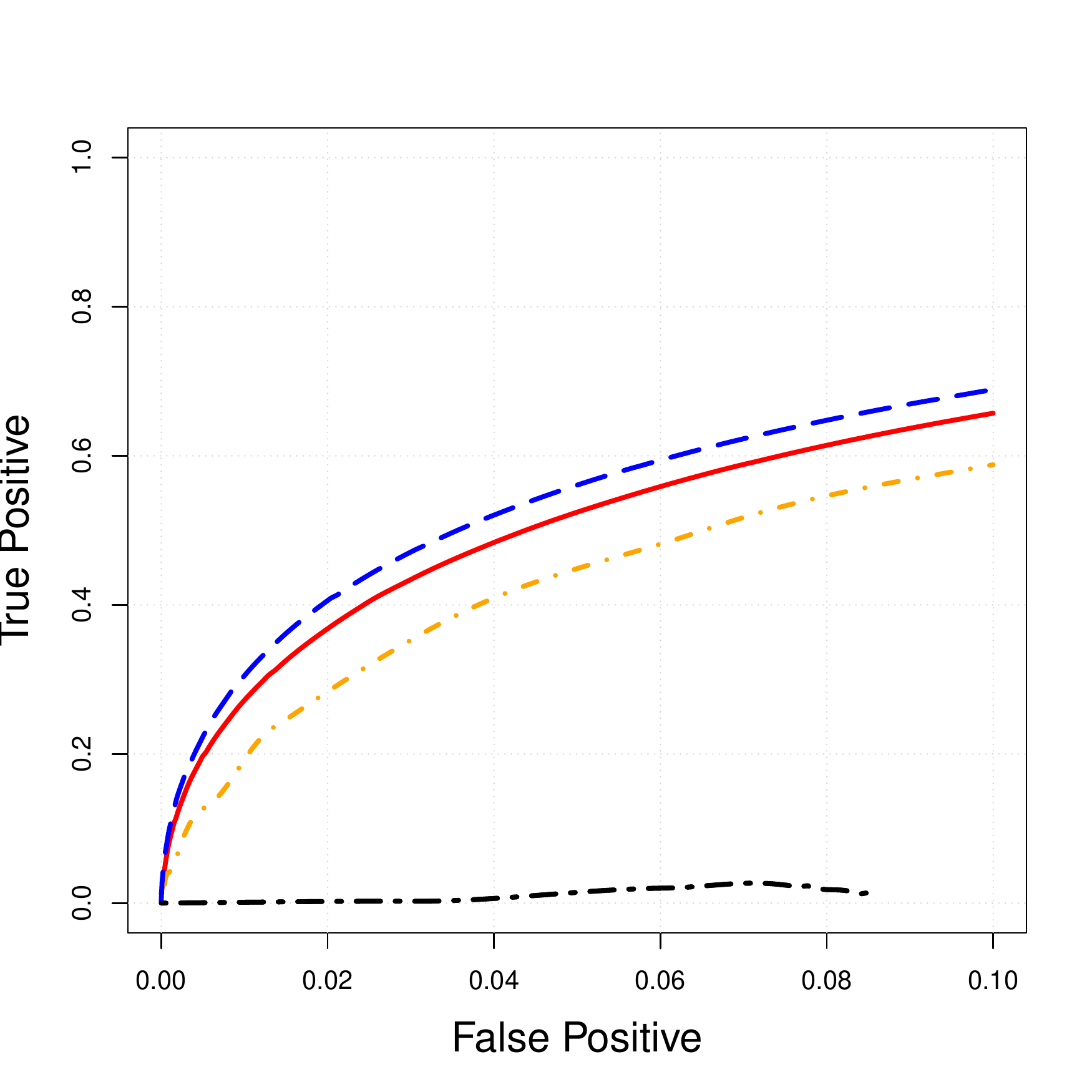}
  \caption{Four connected components.}
\end{subfigure}
\begin{subfigure}{.33\textwidth}
  \centering
  \includegraphics[width=\linewidth]{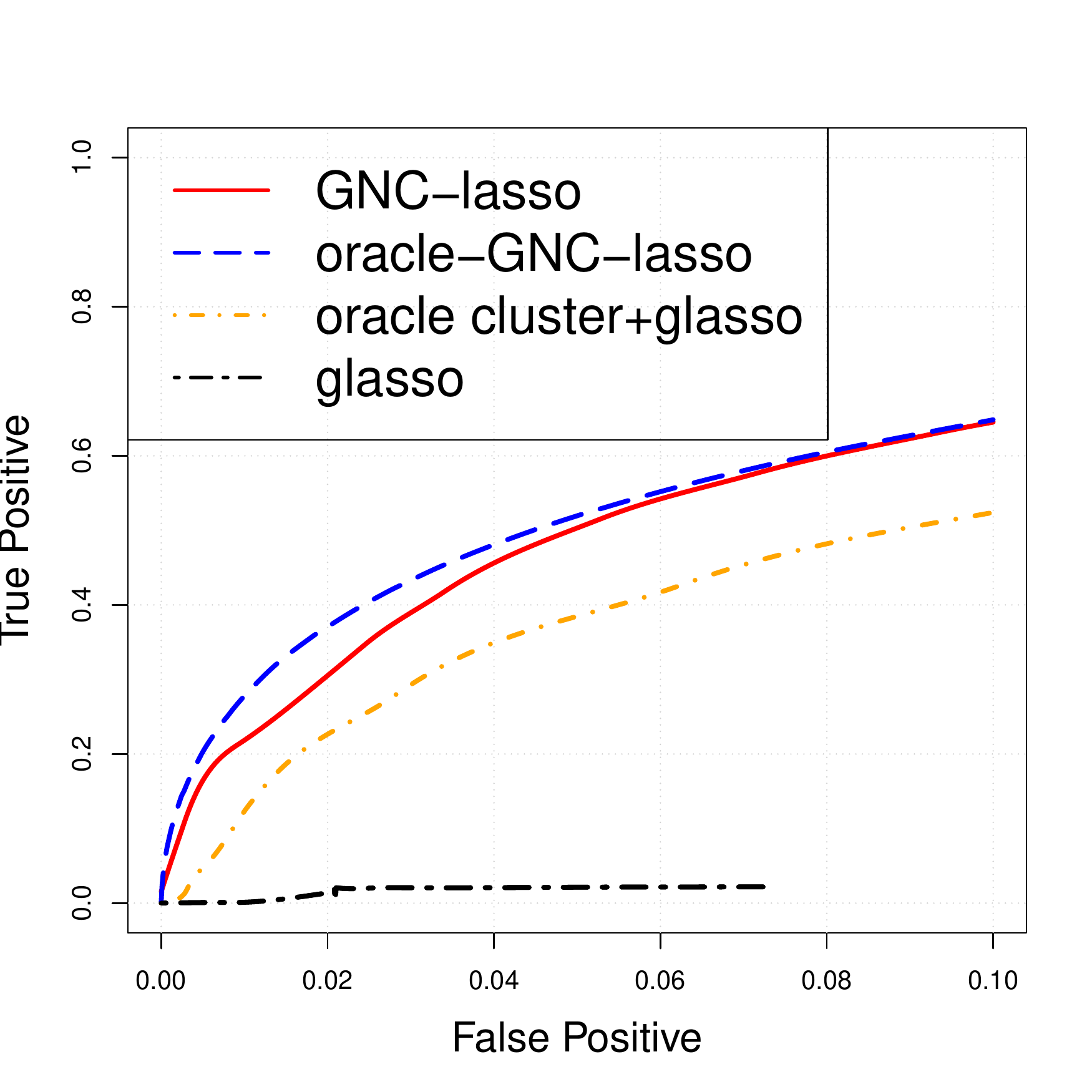}
  \caption{16 connected components.}
\end{subfigure}

\caption{Graph recovery ROC curves for the lattice network and two of its sparsified variants. Here $n=400, p=500$ and we set $t=0.5$ in generating $M$.  }
\label{fig:Lattice-Disconnected}
\end{figure}

\begin{table}[ht]
\centering
\caption{Mean estimation errors for the four methods, averaged over 100 replications, with the lowest error in each configuration indicated in bold, for the lattice networks with one, four, and 16 connected components. }
\label{tab:M-estimate-disconnected-grid}
\begin{tabular}{r|rrr|rrr}
  \hline
   & \multicolumn{3}{c}{$\norm{\hat{M}-M}_{\infty}$}& \multicolumn{3}{c}{$\norm{\hat{M}-M}_{2,\infty}$}\\
  \hline
 method & original & 4 comp. & 16 comp. & original & 4 comp. & 16 comp.  \\ 
  \hline
 glasso& 0.746 & 2.801 & 2.942  & 5.819 & 40.25 & 25.16   \\ 
oracle cluster+glasso &  0.539 & 1.091 & 1.099 & 3.293 & 12.20 & 8.22 \\
oracle GNC-lasso & {\bf 0.520} &{\bf  0.866} & {\bf 0.785} & {\bf 2.054} &{\bf 6.46} & {\bf 5.13} \\
GNC-lasso  &0.669 & 0.983 & 0.838 & 2.619 & 6.73 &5.79 \\ 
   \hline
\end{tabular}
\end{table}

\begin{figure}[H]
\centering
\begin{subfigure}{.33\textwidth}
  \centering
  \includegraphics[width=\linewidth]{./Figures/R1-CoauthorNet-Mix050-SNC060}
  \caption{Original coauthor-network}
\end{subfigure}%
\begin{subfigure}{.33\textwidth}
  \centering
  \includegraphics[width=\linewidth]{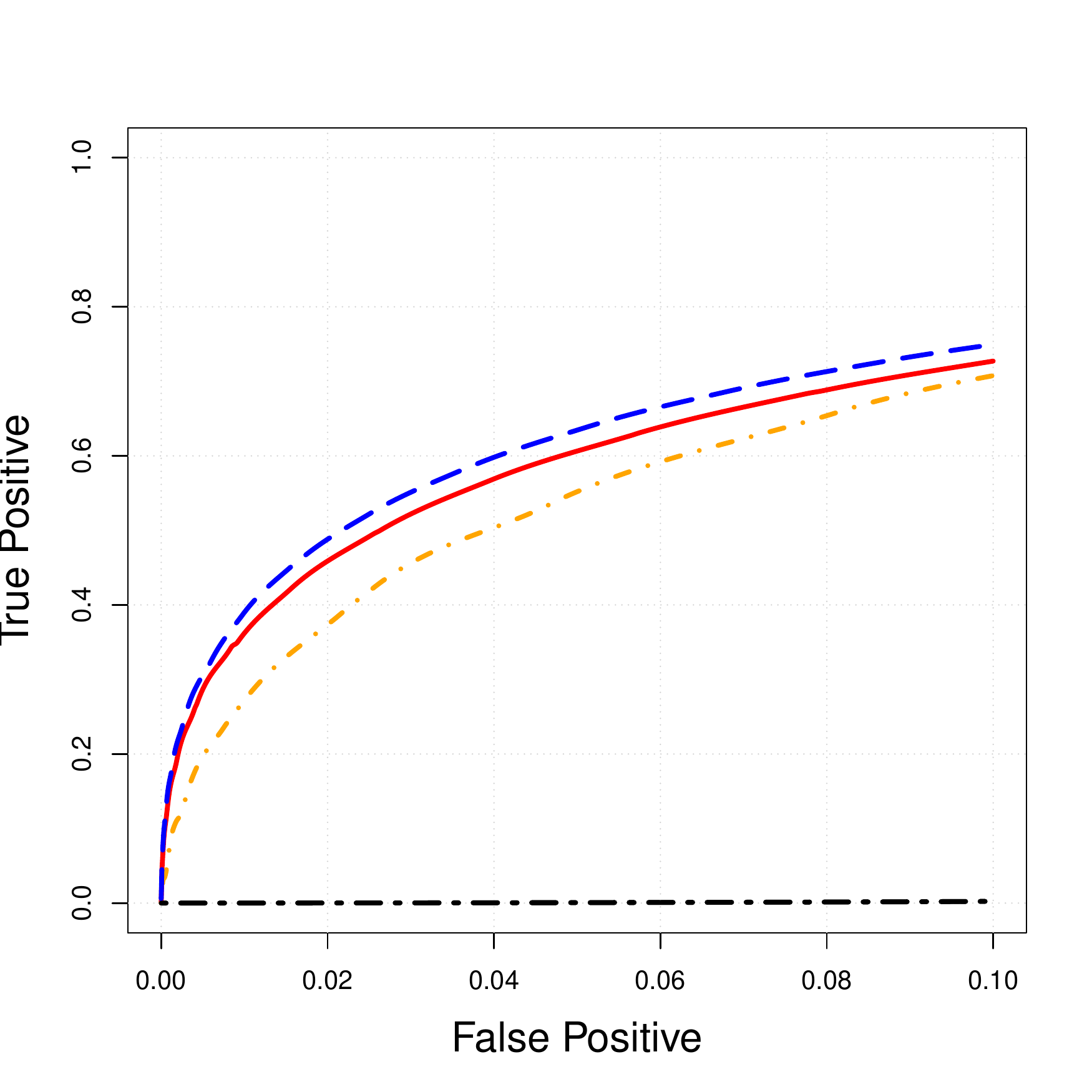}
  \caption{2 connected components.}
\end{subfigure}
\begin{subfigure}{.33\textwidth}
  \centering
  \includegraphics[width=\linewidth]{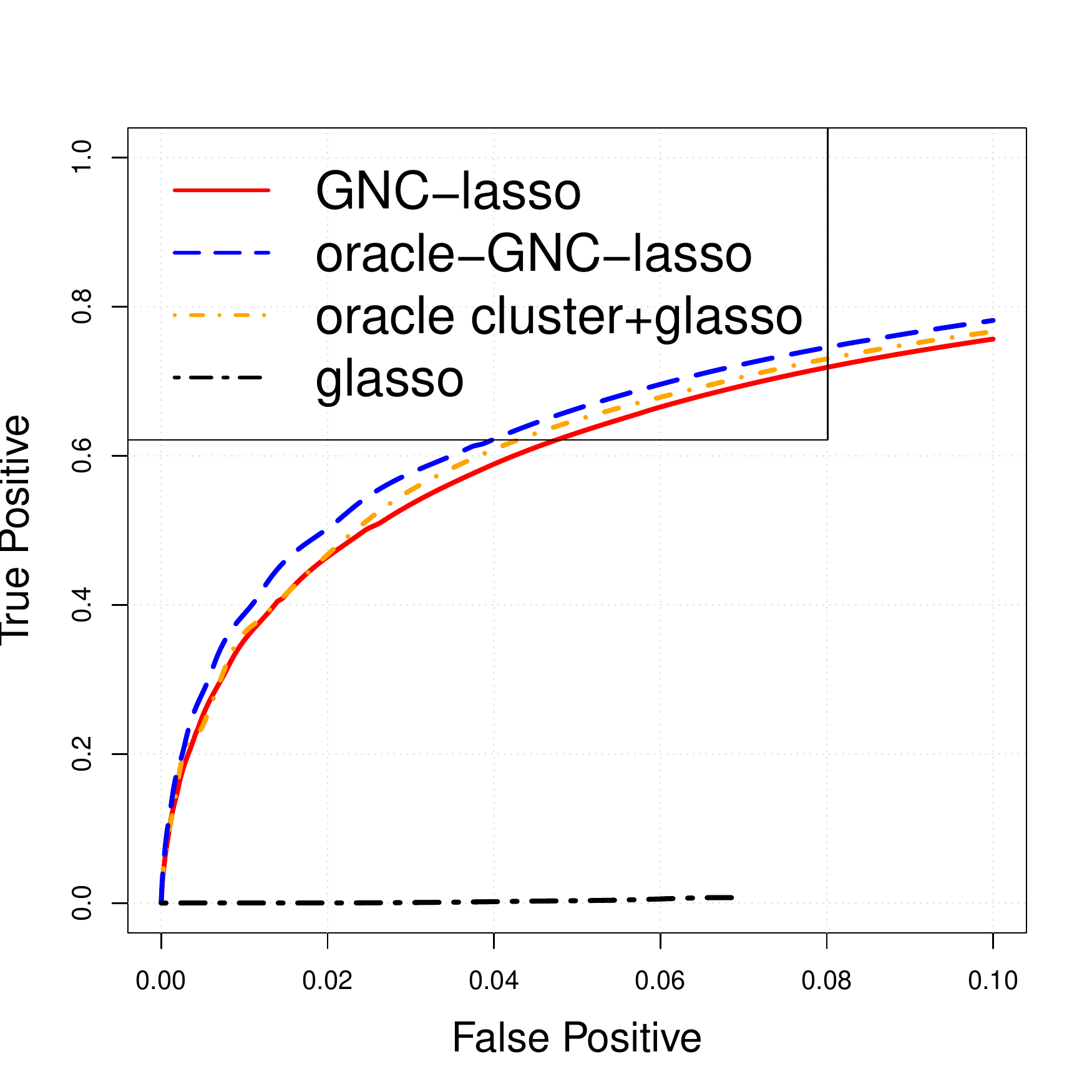}
  \caption{4 connected components.}
\end{subfigure}

\caption{Graph recovery ROC curves for the coauthorship network and two of its sparsified variants. Here $n=635, p=800$ and we set $t=0.5$ in generating $M$.   }
\label{fig:Coauthor-Disconnected}
\end{figure}

\begin{table}[ht]
\centering
\caption{Mean estimation errors for the four methods, averaged over 100 replications, with the lowest error in each configuration indicated in bold, for the coauthorship networks with one, two, or four components.}
\label{tab:M-estimate-disconnected-coauthor}
\begin{tabular}{r|rrr|rrr}
  \hline
   & \multicolumn{3}{c}{$\norm{\hat{M}-M}_{\infty}$}& \multicolumn{3}{c}{$\norm{\hat{M}-M}_{2,\infty}$}\\
  \hline
 method & original & 2 comp. & 4 comp. & original & 2 comp. & 4 comp.  \\ 
  \hline
 glasso& 0.746 & 1.949 & 4.019  & 5.819 & 21.214 & 32.307   \\ 
oracle cluster+glasso &  0.539 & 0.936 & 1.676 & 3.293 & 6.033 & 7.208 \\
oracle GNC-lasso & {\bf 0.520} &{\bf  0.659} & {\bf 0.958} & {\bf 2.054} &{\bf 3.860} & {\bf 4.843} \\
GNC-lasso  &0.669 & 0.852 & 1.289 & 2.619 & 4.947 & 5.102 \\ 
   \hline
\end{tabular}
\end{table}

\subsection{Performance as a function of the sample size}

Here we compare the methods when the sample size $n$ changes while $p$ remains fixed. Specifically, we compare $10 \times 10$, $15 \times 15$, and $20\times 20$ lattices, corresponding to $n=100$, $225$, and $400$, respectively. The dimension $p =500$, the data generating mechanism, and $t = 0.5$ remain the same as in Section~\ref{secsec:cohesion-eval}.     When $n=100$, the sample size is too small for 10-fold cross-validation to be stable, and thus we use leave-one-out cross-validation instead. Figure~\ref{fig:sim-samplesize} shows the ROC curves while Table~\ref{tab:samplesize} shows errors in the mean.   Clearly, the problem is more difficult for smaller sample sizes, but both versions of GNC-lasso still work better than the other two baseline methods, even though for $n =100$, the problem is essentially too difficult for all the methods.  Results on estimating the mean do not favor any one method clearly, but the differences between the methods are not very large in most cases.

\begin{figure}[H]
\centering
\begin{subfigure}{.33\textwidth}
  \centering
  \includegraphics[width=\linewidth]{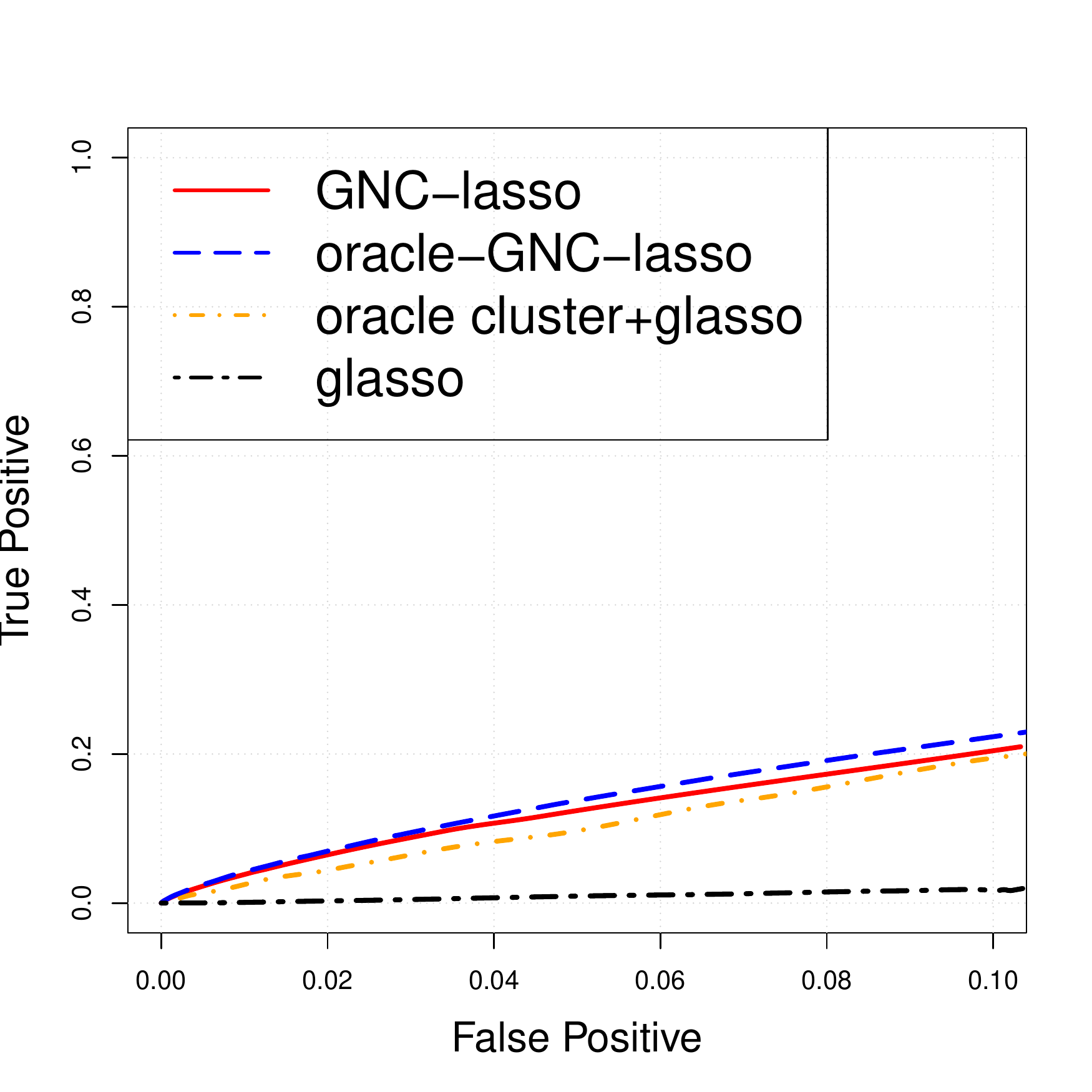}
  \caption{$10\times 10$ lattice.}
\end{subfigure}%
\begin{subfigure}{.33\textwidth}
  \centering
  \includegraphics[width=\linewidth]{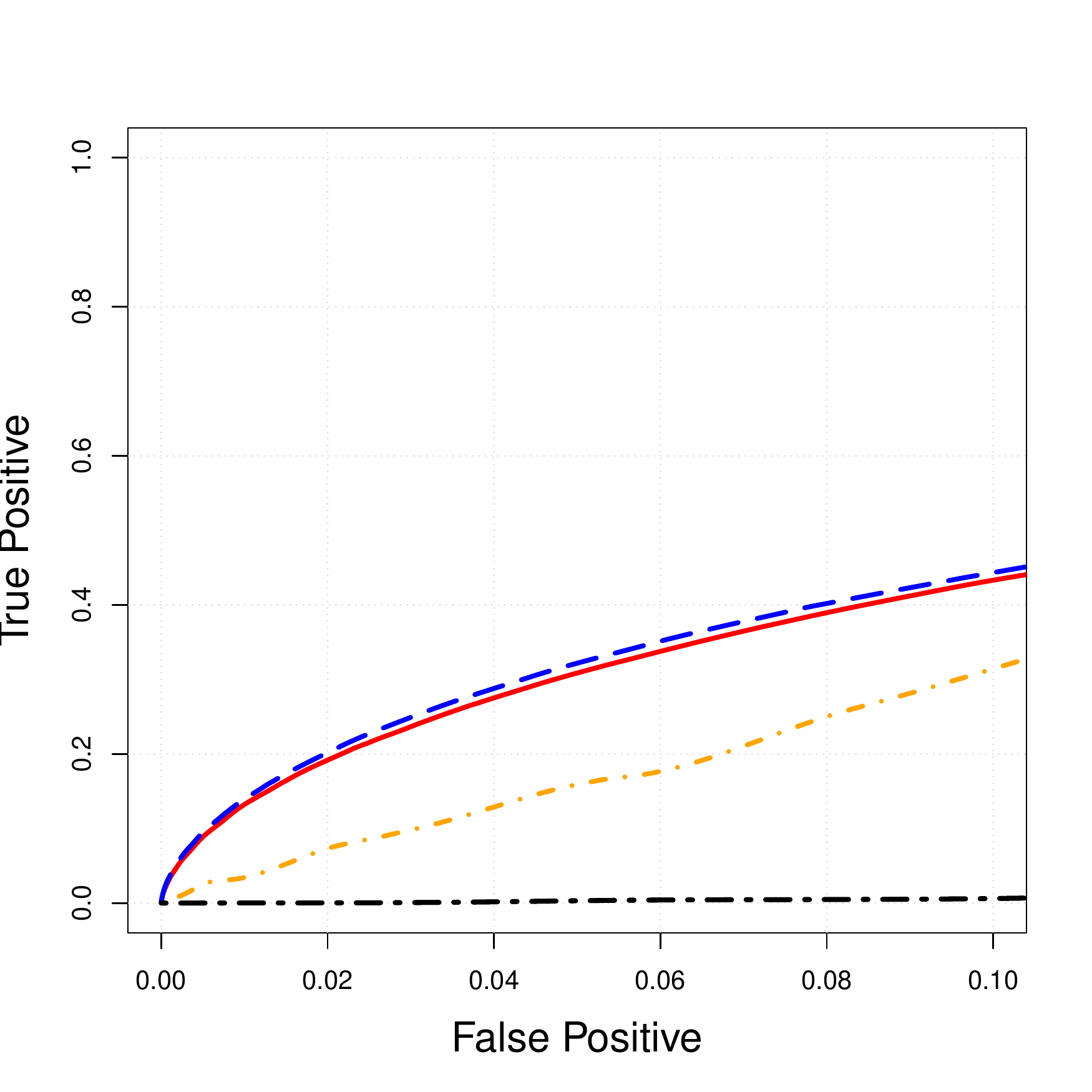}
  \caption{$15\times 15$ lattice.}
\end{subfigure}
\begin{subfigure}{.33\textwidth}
  \centering
  \includegraphics[width=\linewidth]{./Figures/R1-GridNet-Mix050-SNC060}
  \caption{$20\times 20$ lattice.}
\end{subfigure}

\caption{Graph recovery ROC curves for the three lattice networks with $n=100~(10\times 10)$, $225~(15\times 15)$ and $400~(20\times 20)$.  We fix $p=500$, $t=0.5$.  }
\label{fig:sim-samplesize}
\end{figure}

\begin{table}[ht]
\centering
\caption{The estimation errors of $M$ from the four methods on three connected lattice networks with varying sample size, averaged over 100 independent replications. The network sizes are 100, 225 and 400, corresponding to lattice dimension $10\times 10$, $15\times 15$ and $20\times 20$, respectively.}
\label{tab:samplesize}
\begin{tabular}{r|rrr|rrr}
  \hline
   & \multicolumn{3}{c}{$\norm{\hat{M}-M}_{\infty}$}& \multicolumn{3}{c}{$\norm{\hat{M}-M}_{2,\infty}$}\\
  \hline
 method & $n=100$ & $225$ & $400$  &$100$ & $225$& $400$  \\ 
  \hline
 glasso& 1.009 & 1.030 & 0.746 & 15.02 & 15.69 & 12.985 \\
oracle cluster+glasso & 0.991 & {\bf 0.716} & 0.539 & 6.62 & {\bf 4.72} & 3.639 \\
oracle GNC-lasso & 0.911 & 0.794 & {\bf 0.520} & 5.52 & 4.83 & {\bf 3.247} \\ 
GNC-lasso  & {\bf 0.874} & 0.988 & 0.669 & {\bf 5.36} & 5.67 & 4.105 \\ 
   \hline
\end{tabular}
\end{table}

\subsection{Comparing with the iterative GNC-lasso}

Finally, we compare the estimator obtained by iteratively optimizing $\Theta$ and $M$ in \eqref{eq:obj1} (iterative GNC-lasso)  to the proposed two-stage estimator (GNC-lasso).   As mentioned in Section~\ref{secsec:joint}, the iterative method is too computationally intensive to tune by cross-validation, so we only compare the oracle versions of both methods,  on the synthetic data used in Section~\ref{secsec:cohesion-eval} with moderate cohesion level $t=0.5$.  The results are shown in Figure~\ref{fig:sim-oracel-joint} and Table~\ref{tab:M-estimate-iterative}.     The methods are essentially identically on the lattice network and the two-stage method is in fact slightly better on the co-author network, indicating that there is no empirical reason to invest in the computationally intensive iterative method.

\begin{figure}[H]
\centering
\begin{subfigure}{.5\textwidth}
  \centering
  \includegraphics[width=\linewidth]{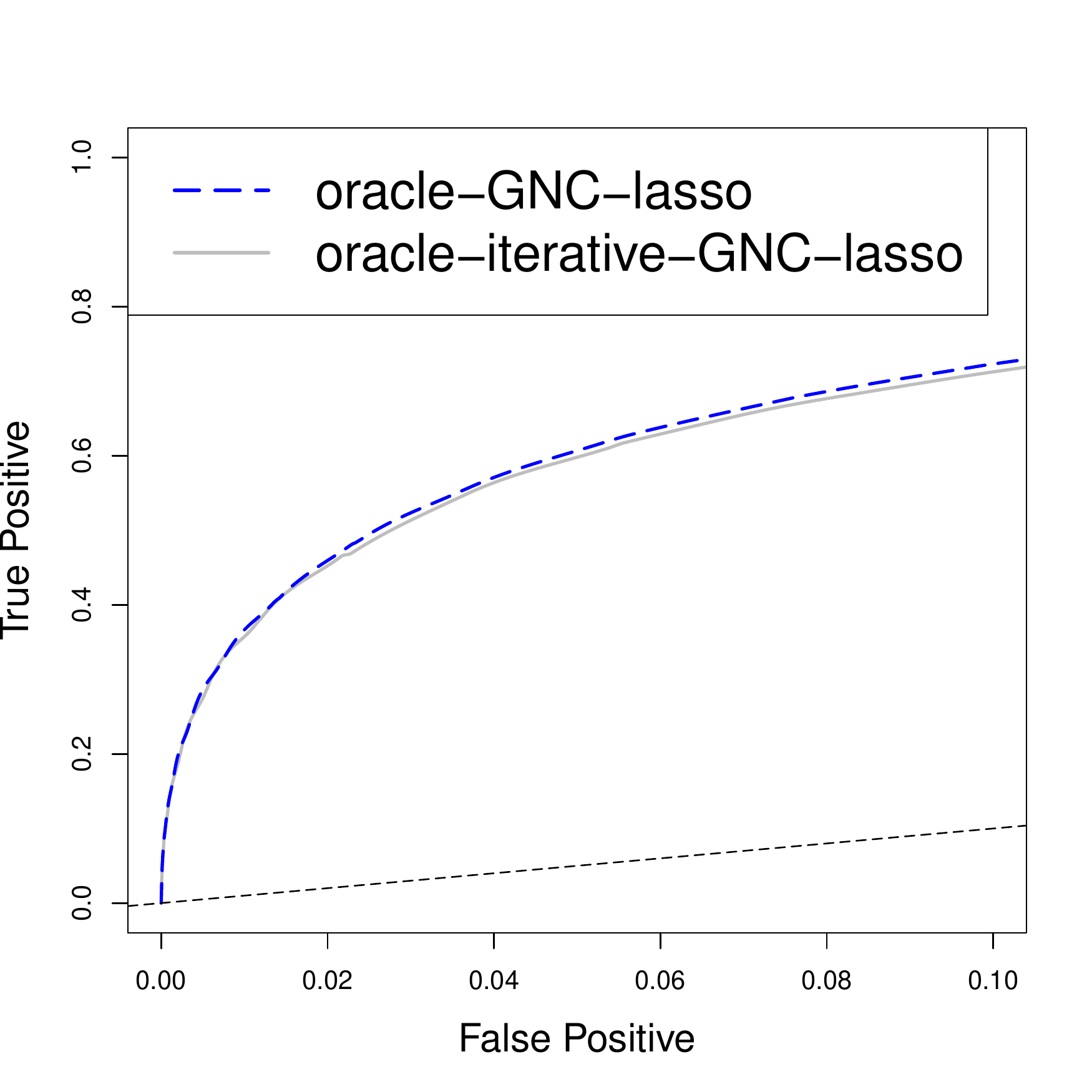}
  \caption{$20\times 20$ lattice}
\end{subfigure}%
\begin{subfigure}{.5\textwidth}
  \centering
  \includegraphics[width=\linewidth]{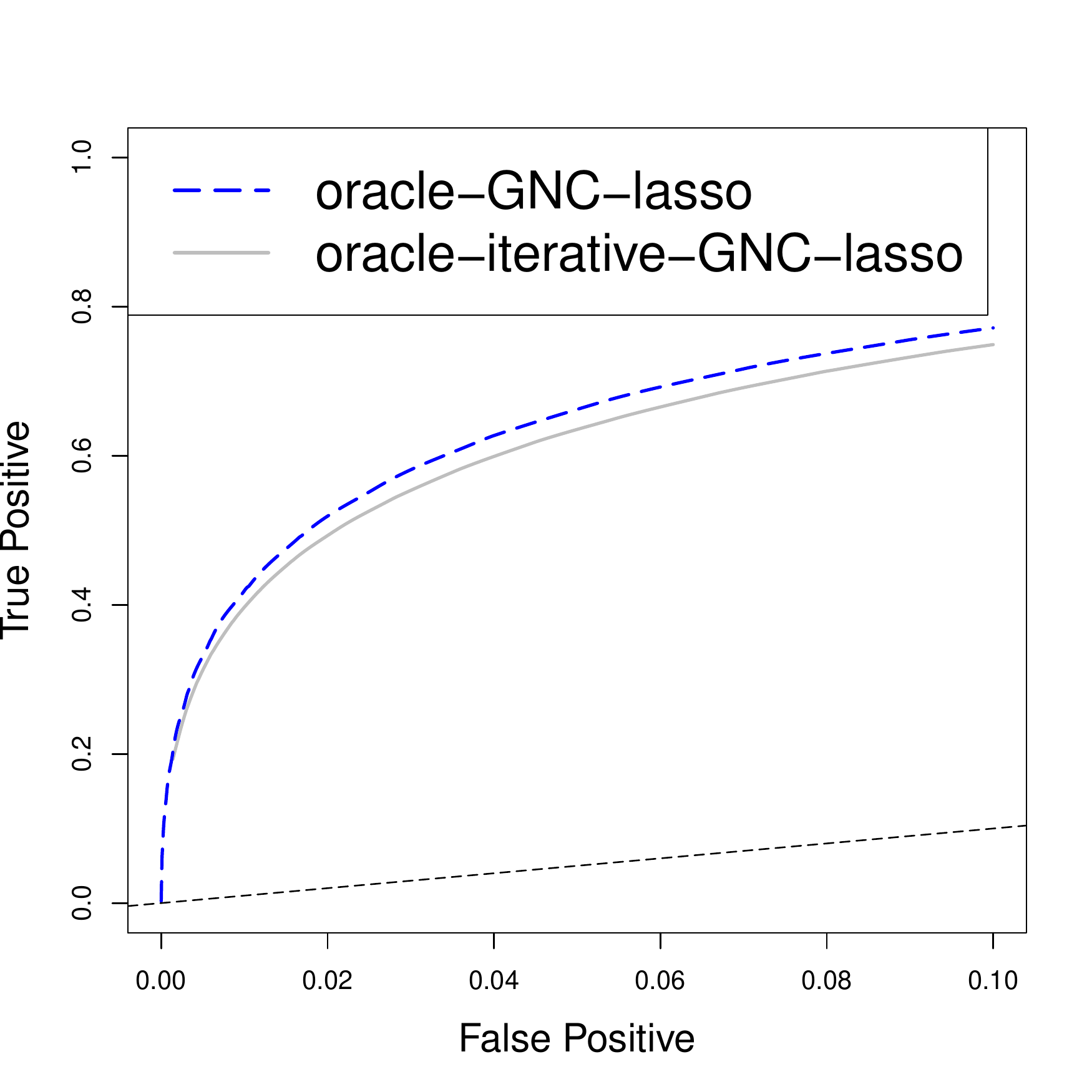}
  \caption{Coauthorship network}
\end{subfigure}
\caption{Graph recovery ROC curves for the proposed two-stage GNC-lasso and the joint GNC-lasso. The network cohesion corresponding to  $t=0.5$ for the $20\times 20$ lattice and the coauthorship network. }
\label{fig:sim-oracel-joint}
\end{figure}

\begin{table}[ht]
\centering
\caption{The estimation errors of $M$ from the iterative and two-stage oracle GNC-lasso methods, averaged over 100 independent replications. }
\label{tab:M-estimate-iterative}
\begin{tabular}{r|r|r|r}
  \hline
 network &method & $\norm{\hat{M}-M}_{\infty}$& $\norm{\hat{M}-M}_{2,\infty}$\\
  \hline
\multirow{ 2}{*}{lattice} & two-stage  &  0.520  & 3.247  \\ 
&  iterative & 0.587 & 3.639 \\ 
   \hline
\multirow{ 2}{*}{coauthorship} & two-stage &  0.860  & 7.404  \\ 
&iterative & 1.21 &  9.92\\ 
   \hline
\end{tabular}
\end{table}

\section{Data analysis: learning associations between statistical terms}\label{sec:app}

Here we apply the proposed method to the dataset of papers from 2003-2012 from four statistical journals collected by \cite{ji2016coauthorship}.    The dataset contains full bibliographical information for each paper and was curated for disambiguation of author names when necessary.  Our goal is to learn a conditional dependence graph between terms in paper titles, with the aid of the coauthorship network.

\begin{figure}[H]
\vspace{-1.5cm}
\begin{center}
\centerline{\includegraphics[width=0.8\columnwidth]{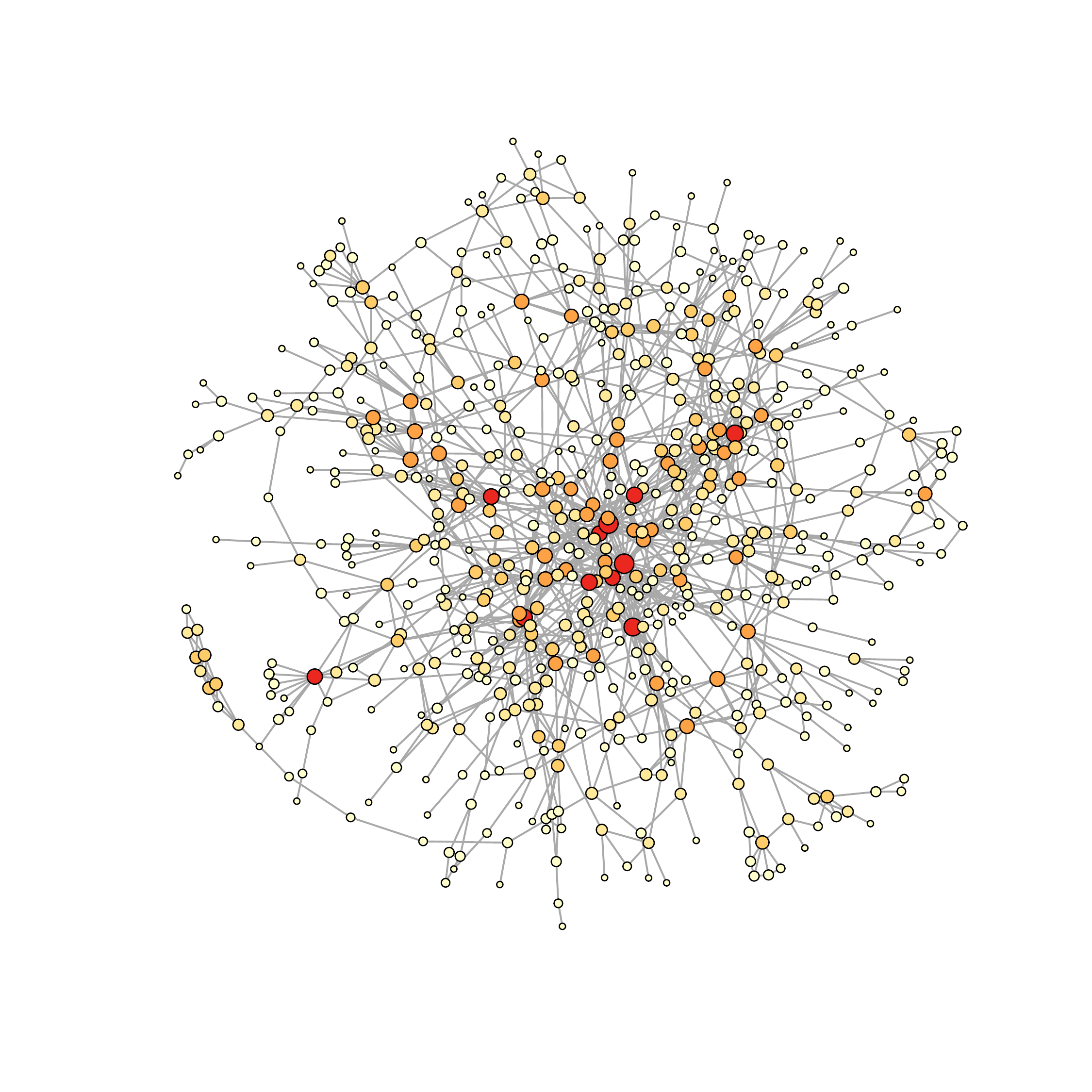}}
\vspace{-2cm}
\caption{The coauthorship network of 635 statisticians (after pre-processing).  The size and the color of each node correspond to the degree (larger and darker circles have more connections.}
\label{fig:CoauthorNet}
\end{center}
\end{figure} 
\vspace{-1.5cm}

We pre-processed the data by removing authors who have only one paper in the data set,  and filtering out common stop words (``and'', ``the'', etc)  as well as terms that appear in fewer than 10 paper titles.    We then calculate each author's  average term frequency across all papers for which he/she is a coauthor.  Two authors are connected in the coauthorship network if they have co-authored at least one paper, and we focus on the largest connected component of the network. Finally, we sort the terms according to their term frequency-inverse document frequency score (tf-idf), one of the most commonly used measures in natural language processing to assess how informative a term is \citep{leskovec2014mining}, and keep 300 terms with the highest tf-idf scores.  After all pre-processing, we have $n=635$ authors and $p=300$ terms.  The observations are 300-dimensional vectors recording the average frequency of term usage for a specific author. The coauthorship network is shown in Figure~\ref{fig:CoauthorNet}.

The interpretation in this setting is very natural;  taking coauthorship into account makes sense in estimating the conditional graph, since the terms come from the shared paper title.    We can expect that there will be standard phrases that are fairly universal  (e.g., ``confidence intervals''), as well as phrases specific to relatively small groups of authors with multiple connections, corresponding to specific research area (e.g., ``principal components'' ), which is exactly the scenario where our model should be especially useful relative to the standard Gaussian graphical model.  To ensure comparable scales for both columns and rows, we standardize the data using the successive normalization procedure introduced by \cite{olshen2010successive}.    If we select $\alpha$ using 10-fold cross-validation, as before, the graphs from GNC-lasso and glasso recover 4 and 6 edges, respectively, which are very sparse graphs.   To keep the graphs comparable and to allow for more interpretable results, we instead set the number of edges to 25 for both methods, and compare resulting graphs, shown in = Figure~\ref{fig::realglasso} (glasso) and Figure~\ref{fig::realGNC} (GNC-glasso).   For visualization purposes, we only plot the 55 terms that have at least one edge in at least one of the graphs.

Overall, most edges recovered by both methods represent common phrases in the statistics literature, including ``exponential families'', ``confidence intervals'', ``measurement error'', ``least absolute''  (deviation), ``probabilistic forecasting'', and ``false discovery''.   There are many more common phrases that are recovered by GNC-lasso but missed by Glasso, for example,  ``high dimension(al/s)'', ``gene expression'', ``covariance matri(x/ces)'', ``partially linear'', ``maximum likelihood'', ``empirical likelihood'', ``estimating equations'', ``confidence bands'', ``accelerated failure'' (time model),``principal components'' and ``proportional hazards''.   There are also a few that are found by Glasso but missed by GNC-lasso, for example,  ``moving average'' and ``computer experiments''.     Some edges also seem like potential false positives, for example, the links between ``computer experiments'' and ``orthogonal construction'', or the edge between ``moving average'' and ``least absolute'', both found by glasso but not GNC-lasso.

Additional insights about the data can be drawn from the  $\hat{M}$ matrix estimated by GNC-lasso;  glasso does not provide any information about the means.   Each  $\hat{M}_{\cdot j}$ can be viewed as the vector of authors' preferences for the term $j$, we can visualize the relative distances between terms as reflected in their popularity.     Figure~\ref{fig:termPC-12} shows the 55 terms from Figure~\ref{fig::realglasso}, projected down from $\hat{M}$ to $R^2$ for visualization purposes by multidimensional scaling (MDS) \citep{mardia1978some}.      The visualization shows a clearly outlying cluster, consisting of the terms ``computer'', ``experiments'', ``construction'', and ``orthogonal'', and to a lesser extent the cluster ``Markov Chain Monte Carlo'' is also further away from all the other terms.    The clearly outlying group can be traced back to a single paper,  with the title ``Optimal and orthogonal Latin hypercube designs for computer experiments" \citep{butler2001optimal}, which is the only title where the words ``orthogonal'' and ``experiments'' appear together.   Note that glasso estimated them as a connected component  in the graph, whereas GNC-lasso did not, since it was able to separate a one-off combination occurring in a single paper from a common phrase.   This illustrates the advantage of GNC-lasso's ability to distinguish between individual variation in the mean vector and the overall dependence patterns, which glasso lacks.

\section{Discussion}\label{sec:con}
We have extended the standard graphical lasso problem and the corresponding estimation algorithm to the more general setting in which each observation can have its own mean vector.  We studied the case of observations connected by a network and leveraged the empirically known phenomenon of network cohesion to share information across observations, so that we can still estimate the means in spite of having $np$ mean parameters instead of just $p$ in the standard setting.   The main object of interest is the inverse covariance matrix, which is shared across observations and represents universal dependencies in the population.   while all observations share the same covariance matrix under the assumption of network cohesion.    The method is computationally efficient with theoretical guarantees on the  estimated inverse covariance matrix and the corresponding graph.    Both simulations and an application to a citation network show that GNC-lasso is more accurate and gives more insight into the structure of the data than the standard glasso when observations are connected by a network.
One possible avenue for future work is obtaining inference results for the estimated model. This might be done by incorporating the inference idea of \cite{zhao2016significance} and \cite{ren2015asymptotic} with additional structural assumptions on the mean vectors.  The absolute deviation penalty \citep{hallac2015network} between connected nodes is a possible alternative, if the computational cost issue can be resolved through some efficient optimization approach.   Another direction is to consider the case where the partial dependence graphs themselves differ for individuals over the network, but in a cohesive fashion;  the case of jointly estimating several related graphs has been studied by \cite{guo2011joint, danaher2014joint}.  As always, in making the model more general there will be a trade-off between goodness of fit and parsimony, which may be elucidated by obtaining convergence rates in this setting.

\begin{figure}[H]
  \centering
  \begin{minipage}{1\linewidth}
    \centering
      \vspace{-4.5cm}
    \includegraphics[width=1.2\linewidth]{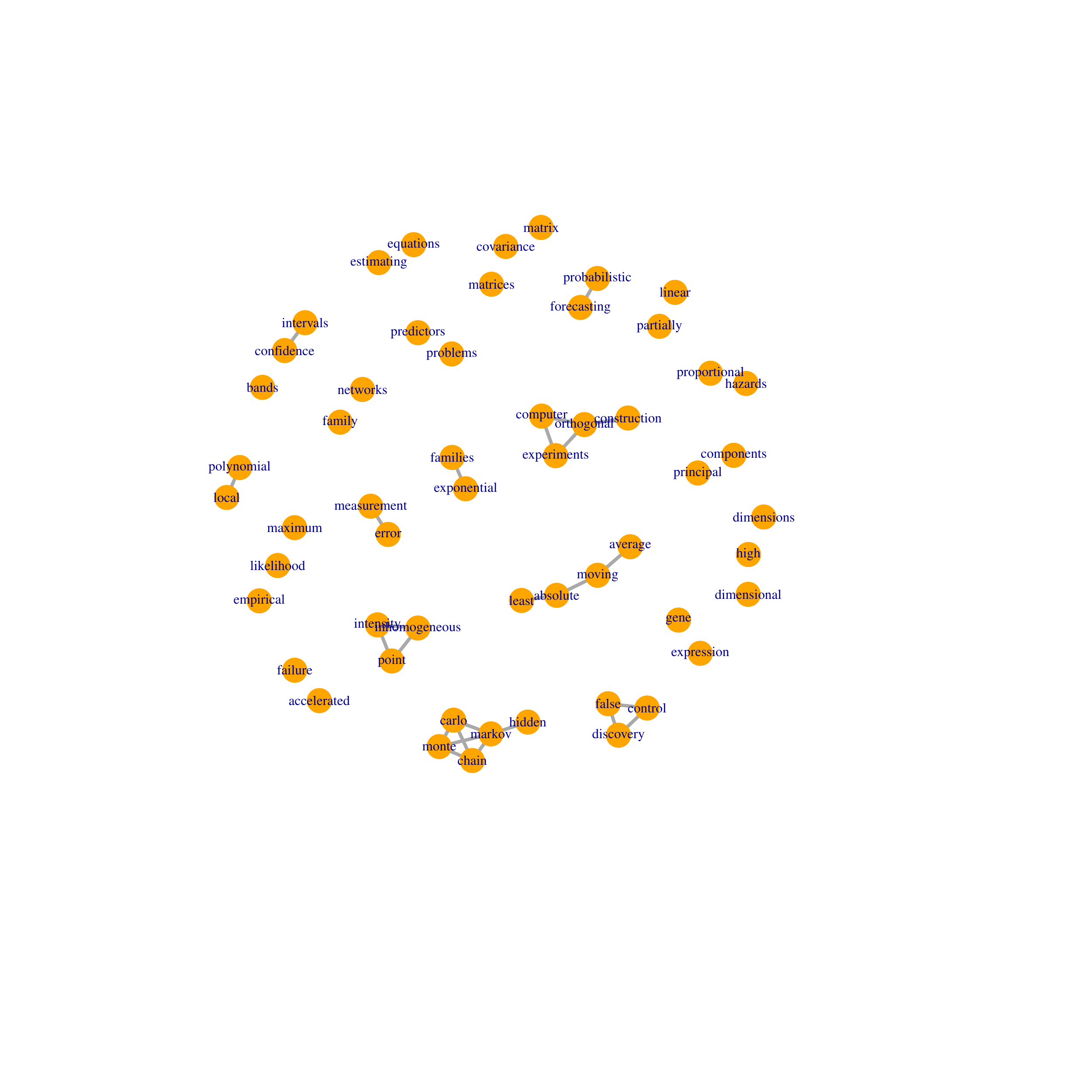}
     \vspace{-6cm}
      \caption{Partial correlation graphs estimated by Glasso}\label{fig::realglasso}
    
  \end{minipage}\\
  \begin{minipage}{1\linewidth}
    \centering
     \vspace{-3.5cm}
   \includegraphics[width=1.2\linewidth]{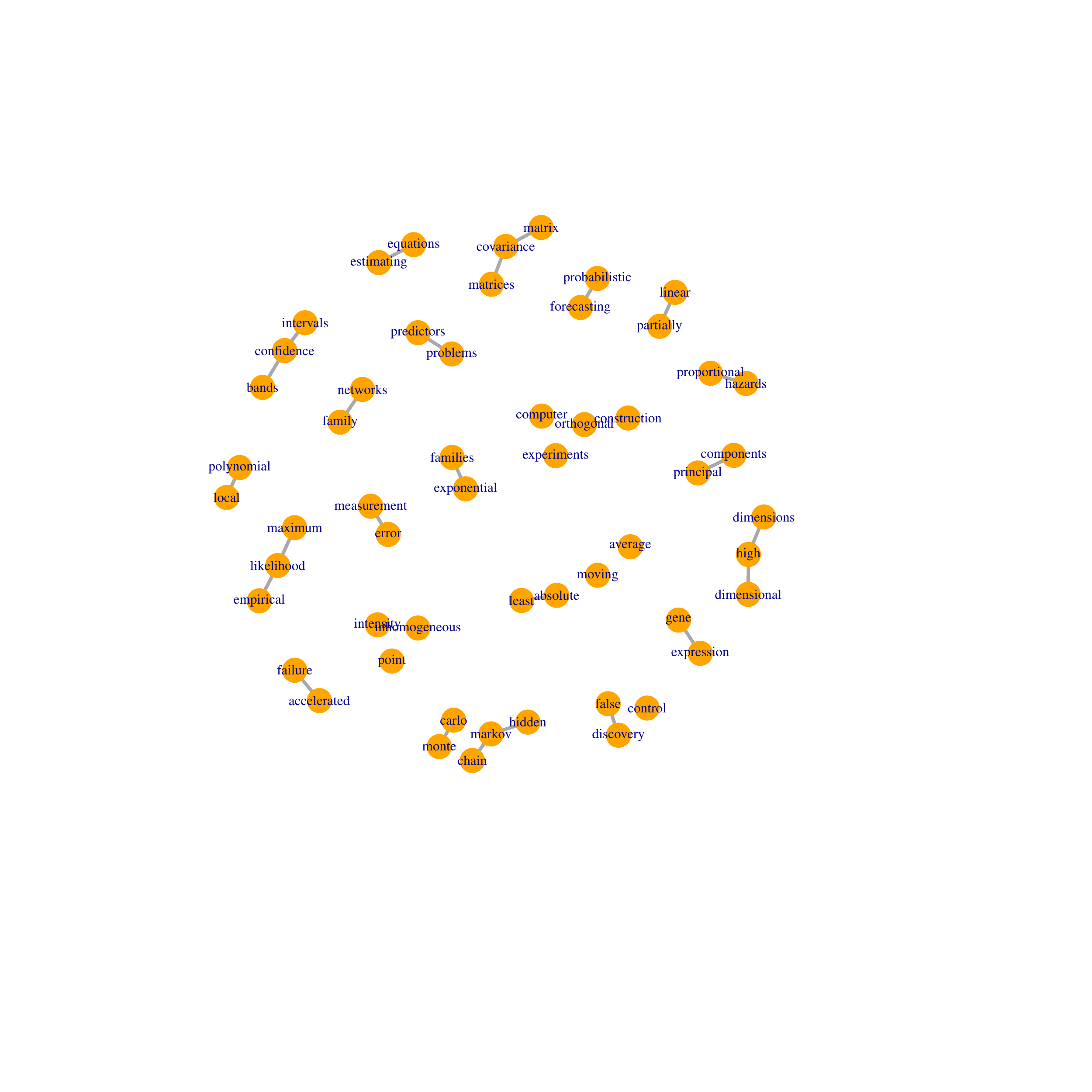}
    \vspace{-6cm}
        \caption{Partial correlation graphs estimated by GNC-lasso}\label{fig::realGNC}
    \end{minipage}  
  \label{fig:StaisticalNet}
\end{figure}

\begin{figure}[H]
\vspace{-0.5cm}
\begin{center}
\centerline{\includegraphics[width=0.9\columnwidth]{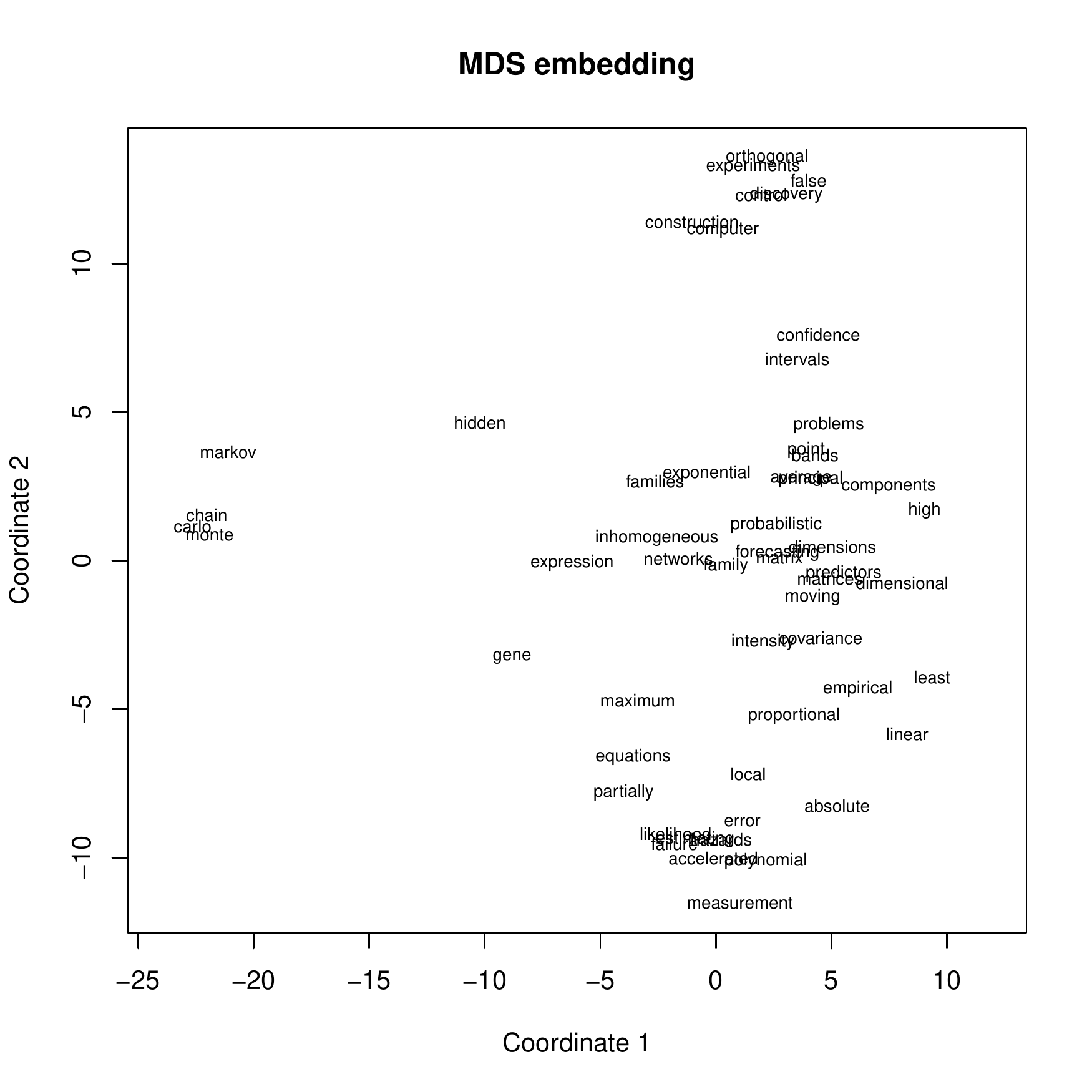}}
\caption{Projection of 55 terms by using the 2-D MDS.}
\label{fig:termPC-12}
\end{center}
\end{figure}

\acks{
T. Li was partially supported by the Quantitative
Collaborative grant from the University of Viriginia. E. Levina and T. Li (while a PhD student at the University of Michigan) were supported in part
by an ONR grant (N000141612910) and NSF grants (DMS-1521551 and DMS-1916222). J. Zhu and T. Li (while a
PhD student at the University of Michigan) were supported in part by NSF grants (DMS-1407698
and DMS-1821243). C. Qian was supported in part by grants from National Key R\&D Program of China
(2018YFF0215500) and Science Foundation of Jiangsu Province for Young Scholars (SBK2019041494). We want to thank the action editor and reviewers for their valuable suggestions.}

\bibliography{CommonBib}{}

\newpage

\begin{appendix}
\section{Proofs}
First, recall the following matrix norm definitions we'll need: for any matrix $M$,  $\norm{M}_{\infty} = \max_{ij}|M_{ij}|$, 
$\norm{M}_{1,1} = \max_{j}\norm{M_{\cdot j}}_1$, and  
$\norm{M}_{\infty,\infty} = \max_{i}\norm{M_{i\cdot}}_1.$

The following lemma summarizes a few concentration inequalities that we will need.

\begin{lem}[Concentration of norm of a multivariate Gaussian]\label{lem:concentration}
  For a Gaussian random vector $x \sim \ncal(0,\Sigma)$, with $\Sigma \in \bR^{p\times p}$ a positive definite matrix     and $\phi_{\max}(\Sigma)$ the largest eigenvalue of $\Sigma$, we have,
  \begin{eqnarray}
   \label{eq:1stConcentration}
\p(|\norm{x}_2 - \sqrt{\tr(\Sigma)}|>t) & \le & 2\exp(-c\frac{t^2}{\phi_{\max}(\Sigma)}),  \\
\label{eq:2ndConcentration}
\p(|\norm{x}_2^2 - \tr(\Sigma)|>t) & \le & 2\exp(-c\frac{t}{\phi_{\max}(\Sigma)}), \\
\label{eq:3rdConcentration}
\p(|\norm{x}_1 - \frac{2}{\pi}\sum_{i=1}^p\sqrt{\Sigma_{ii}}|>t) &\le & 2\exp(-c\frac{t^2}{p\phi_{\max}(\Sigma)})
\end{eqnarray}
for some generic constant $c>0$.  Further, if $x_1, \cdots, x_n$ are i.i.d.\ observations from $\ncal(0,\Sigma)$, then 
\begin{equation}\label{eq:4thConcentration}
\p(\sum_i^n\norm{x_i}_2^2 > 2n\tr(\Sigma)) \le 2\exp(-cnr(\Sigma))
\end{equation}
where $r(\Sigma)$ is the stable rank of $\Sigma$.
\end{lem}
\begin{proof}[Proof of Lemma~\ref{lem:concentration}]
The first inequality \eqref{eq:1stConcentration} follows from concentration of a Lipschitz function of a sub-Gaussian random vector.     Inequalities \eqref{eq:2ndConcentration} and \eqref{eq:3rdConcentration} follow from the definition of a sub-exponential random variable. Lastly, \eqref{eq:4thConcentration} follows from applying Bernstein's inequality to \eqref{eq:2ndConcentration} with $t = n\tr(\Sigma)$.  
\end{proof}

\begin{proof}[Proof of Proposition~\ref{prop:lattice}]
By \cite{edwards2013discrete}, the eigenvalues of $A$ are given by 
\begin{equation}\label{eq:eigenvalue}
\frac{1}{\bar{d}} (4\sin^2(\frac{i\pi}{2\sqrt{n}})+4\sin^2(\frac{j\pi}{2\sqrt{n}})), i,j \in\{0,1, \cdots, \sqrt{n}-1\}.
\end{equation}
Since the average degree $2 \le \bar{d} \le 4$ for a lattice network, we ignore this constant.   First, we show $m_A \le n^{2/3}$, which by definition of $m_A$ is equivalent to  $\tau_{n-n^{2/3}} \ge n^{-1/3}$.  
Define the set of all eigenvalues satisfying this condition as
$$\acal_n = \{(i,j): i,j \in \bN \cap [0,\sqrt{n}-1], 4\sin^2(\frac{i\pi}{2\sqrt{n}})+4\sin^2(\frac{j\pi}{2\sqrt{n}}) < n^{-1/3}\}.$$
Then it is sufficient to show $|\acal_n| < n^{2/3}.$
Applying the inequality $\sin(x) \ge \frac{2}{\pi}x$ for $x \in [0,\pi/2]$, we can see that it is sufficient to show $|\tilde{\acal}_n|< n^{2/3}$, where
$$\tilde{\acal}_n  = \{(i,j): i,j \in \bN \cap [0,\sqrt{n}-1], \frac{4i^2}{n}+\frac{4j^2}{n} < n^{-1/3}\}.$$
The cardinality of $\tilde{\acal}_n$ can be computed exactly by counting; for simplicity,  we give an approximate calculation for when  $n$ is sufficiently large.   In this case the proportion of pairs $(i,j)$ out of the entire set of $ (\bN \cap [0,\sqrt{n}-1])\times ( \bN \cap [0,\sqrt{n}-1])$ that satisfy the condition to be included in $\tilde{\acal}_n$ can be upper bounded by twice the ratio betwen the area of the quarter circle with radius $\frac{n^{1/3}}{2}$ and the area of the $\sqrt{n} \times \sqrt{n}$ square. This gives
$$|\tilde{\acal}_n| \le 2 \frac{\pi}{16}n^{2/3} < n^{2/3}.$$

To prove the second claim,  consider the $\mu = U\beta$ such that all the inequalities in \eqref{eq:cohesion} hold as equalities and $\delta=1/2$.   Then, by noting that $P_{\mbone}u_n = u_n$, we have
\begin{equation}\label{eq:prop1eq1}
\norm{\mu - P_{\mbone}\mu}_2^2 = \sum_{i<n}\beta_i^2 = \norm{\mu}_2^2n^{-\frac{2(1+\delta)}{3}-1}\sum_{i <n}\frac{1}{\tau_i^2} = \norm{\mu}_2^2n^{-2}\sum_{i <n}\frac{1}{\tau_i^2} \ . 
\end{equation}
We need a lower bound for $\sum_{i <n}\frac{1}{\tau_i^2}$.  By  \eqref{eq:eigenvalue},
\begin{align*}
\sum_{i <n}\frac{1}{\tau_i^2} &= \sum_{i,j \le \sqrt{n}-1, (i,j) \ne (0,0)}\frac{1}{ (   4\sin^2(\frac{i\pi}{2\sqrt{n}})+4\sin^2(\frac{j\pi}{2\sqrt{n}})  )^2}\\
& > \sum_{1\le i,j \le \sqrt{n}-1}\frac{1}{ (   4\sin^2(\frac{i\pi}{2\sqrt{n}})+4\sin^2(\frac{j\pi}{2\sqrt{n}})  )^2} \\
& = \frac{n}{\pi^2}\sum_{1\le i,j \le \sqrt{n}-1}\frac{1}{ (   4\sin^2(\frac{i\pi}{2\sqrt{n}})+4\sin^2(\frac{j\pi}{2\sqrt{n}})  )^2}\frac{\pi}{\sqrt{n}}\frac{\pi}{\sqrt{n}}\\
& \ge \frac{n}{\pi^2}\sum_{1\le i,j \le \sqrt{n}-1}\frac{1}{ (   4\frac{\pi^2i^2}{4n}+4\frac{\pi^2j^2}{4n}  )^2}\frac{\pi}{\sqrt{n}}\frac{\pi}{\sqrt{n}} ~~~~~~~~~~~~~~~~\text{|~ applying} \sin^2(x) \le x^2 \\
& > \frac{1}{2} \frac{n}{\pi^2}\int_{\frac{\pi}{\sqrt{n}}\le x,y \le \pi}{\frac{1}{(x^2+y^2)^2}dxdy} \hfill ~~~~~~~~~~~~~  \text{|~sum lower bounded by 1/2 of the integral }\\
& > \frac{n}{2\pi^2}\int_{\pi/6}^{\pi/3}{\int_{\frac{\pi}{\sqrt{n}/2}}^{\pi}{\frac{1}{r^3}dr}d\theta}   ~~~~\text{|~ polar coordinates, } \{r \in [\frac{2\pi}{\sqrt{n}}, \pi], \theta \in [\frac{\pi}{6},\frac{\pi}{3}]\} \subset  [\frac{\pi}{\sqrt{n}}, \pi]\times [\frac{\pi}{\sqrt{n}}, \pi]    \\
& = \frac{n}{24\pi^3}(\frac{n}{4}-1).
\end{align*}
%
%
%
Substituting this lower bound for $\sum_{i <n}\frac{1}{\tau_i^2}$ in \eqref{eq:prop1eq1}, for a sufficiently large $n$ we have
$$\norm{\mu - P_{\mbone}\mu}_2^2 = \norm{\mu}_2^2n^{-2}\sum_{i <n}\frac{1}{\tau_i^2} > c\norm{\mu}_2^2.$$ 
Therefore, the $\mu$ we constructed is nontrivially cohesive. 

\end{proof}

We can represent each column of $M$ by taking the basis expansion in $U$, obtaining the basis coefficient matrix  $B = (B_{\cdot 1}, B_{\cdot 2}, \cdots, B_{\cdot p})$ such that $M = UB$. Let  $\hat{B} = U^T\hat{M}$, where $\hat M$ is the estimate \eqref{eq:individualRNC}.  We can view $\hat{B}$ as an estimate of $B$. We first state the error bound for $\hat{B}$ in Lemma~\ref{lem:meanCoefbound}, and the bound for $\hat{M}$ directly follows.

\begin{lem}\label{lem:meanCoefbound}
  Under model \eqref{eq:GNC-GGM} and  Assumption~\ref{ass:regressionCoef}, if $\alpha = n^{\frac{1+\delta}{3}}$, we have
  \begin{enumerate}
\item In maximum norm,
\begin{equation}\label{eq:MeanMaxBound}
\norm{\hat{B}-B}_{\infty} \le C\sigma \left( (\sqrt{\log{pn}}\sqrt{m_A}n^{-\frac{1+\delta}{3}}) \vee \frac{\sqrt{\log(pm_A)} }{1+\Delta} \vee \sqrt{\log(p)}\right)
\end{equation}
with probability at least $1-\exp(-c\log{(p(n-m_A))}) - \exp(-c'\log(pm_A))$ for some constants $C$, $C'$, $c$, $c'$, and $c''$.  
\item In Frobenius norm, 
\begin{equation}\label{eq:MeanFrobeniusBound}
\norm{\hat{B}-B}_{F} \le \sqrt{(b^2+2\sigma^2)p((n-m_A)m_An^{-\frac{2(1+\delta)}{3}}+\frac{m_A}{(1+\Delta)^2}+1)}
\end{equation}
with probability at least $1-\exp(-c''(n-m_A)r(\Sigma)) - \exp(-c''m_Ar(\Sigma))- \exp(-c''r(\Sigma))$. 

\item if $\log{p} = o(n)$ and $m_A = o(n)$, then 
\begin{equation}\label{eq:MeanColumnMaxBound}
\norm{\hat{B}-B}_{1,1} \le C'(b+2\sigma)(\sqrt{m_A} n^{\frac{2-\delta}{3}}+\sqrt{\log{p}}(\frac{m_A}{\Delta+1}+1)).
\end{equation}
with probability at least $1-\exp(-cn) - \exp(-Cm_A\log{p})-\exp(-C\log{p})$. 
\end{enumerate}
\end{lem}

\begin{proof}[Proof of Lemma~\ref{lem:meanCoefbound}]
Solving \eqref{eq:individualRNC}, we can explicitly write out 
$$\hat{B} = (I+\alpha\Lambda)^{-1}B + (I+\alpha\Lambda)^{-1}U^TE = (I+\alpha\Lambda)^{-1}B + (I+\alpha\Lambda)^{-1}\tilde{E}. $$
In particular, for each column $j \in [p]$, the estimate can be written as
$$\hat{B}_{\cdot j} = (I+\alpha\Lambda)^{-1}B_{\cdot j} + (I+\alpha\Lambda)^{-1}U^TE_{\cdot j} = (I+\alpha\Lambda)^{-1}B_{\cdot j} + (I+\alpha\Lambda)^{-1}\tilde{E}_{\cdot j},$$
where $\tilde{E}_{\cdot j} \sim \ncal(0, \sigma^2 I)$. Let $\ical^j$ and $\iical^{j}$ be two $n$ dimensional vectors such that the $i$th element of $\ical^j$ is given by $\frac{\alpha\tau_i}{1+\alpha\tau_i}B_{ij}$ while the $i$th element of $\iical^j$ is given by $\frac{1}{1+\alpha\tau_i}\tilde{E}_{ij}$.
$$\hat{B}_{\cdot j}-B_{\cdot j} = \ical^j + \iical^j.$$
For the element-wise $L_{\infty}$ norm, we have  
\begin{align}\label{eq:Linfty_Bound1}
\norm{\ical^{j}}_{\infty} & \le \max_{i<n}\frac{\alpha}{1+\alpha\tau_i} \max_{i<n}|\tau_iB_{ij}|
    \le \frac{\alpha}{1+\alpha\tau_{n-1}} n^{-\frac{1+\delta}{3}-\frac{1}{2}}\norm{B_{\cdot j}}   
        \le b\cdot \alpha n^{-\frac{1+\delta}{3}} = b. 
\end{align}
where the second inequality is by Definition~\ref{defi:cohesion}.       The term $\iical^j$ can be decomposed into two parts, the first $n-m_A$ elements and the last $m_A$ elements.   For the first $n-m_A$ elements, we have
\begin{align}\label{eq:Linfty_Bound2}
\max_{j\in[p]}\norm{\iical^j_{1:n-m_A}}_{\infty} &\le \max_{j\in[p]}\max_{i\le n-m_A}\frac{1}{1+\alpha \tau_i}\max_{i\le n-m_A}|\tilde{E}_{ij}| = \frac{1}{1+\alpha \tau_{n-m_A}}\max_{i\le n-m_A}\max_{j\in[p]}|\tilde{E}_{ij}| \notag\\
& = \frac{1}{1+\tau_{n-m_A}n^{\frac{1+\delta}{3}}}\max_{j\in[p]}\max_{i\le n-m_A}|\tilde{E}_{ij}| 
    \le \frac{\sqrt{4\sigma^2\log(p(n-m_A))}}{\tau_{n-m_A}n^{\frac{1+\delta}{3}}}\notag\\
& \le \sqrt{4\sigma^2\log(p(n-m_A))}n^{-\frac{1+\delta}{3}}\sqrt{m_A}
\end{align}
by Definition~\ref{defi:highrank}, with probability at least $1-\exp(-c\log(p(n-m_A)))$.   For the remaining $m_A$ elements, with probability at least $1-\exp(-c'\log(pm_A))$, we have
\begin{align}\label{eq:Linfty_Bound3}
\max_{j \in [p]}\norm{\iical^j_{n-m_A+1:n}}_{\infty} & = \max_j \sum_{i >n-m_A}\frac{|\tilde{E}_{ij}|}{1+\alpha\tau_i} \notag\\
& = \max_j \sum_{n-m_A<i < n}\frac{|\tilde{E}_{ij}|}{1+n^{\frac{1+\delta}{3}}\tau_i}|\tilde{E}_{ij}| + \max_j|\tilde{E}_{nj}|\notag\\
& \le \frac{\sqrt{4\sigma^2\log{(pm_A)}}}{1+\Delta} + \sqrt{4\sigma^2\log(p)}.
\end{align}
Combining \eqref{eq:Linfty_Bound1}--\eqref{eq:Linfty_Bound3} leads to \eqref{eq:MeanMaxBound}, since 
\begin{align*}
\norm{\hat{B}-B}_{\infty} &\le \max_{j \in [p]}\norm{\ical^j}_{\infty} + \max_{j\in[p]}\norm{\iical^j_{1:n-m_A}}_{\infty} + \max_{j \in [p]}\norm{\iical^j_{n-m_A+1:n}}_{\infty} \\
& \le b+ \sqrt{4\sigma^2\log{p(n-m_A)}}n^{-\frac{1+\delta}{3}}\sqrt{m_A} + \frac{\sqrt{4\sigma^2\log{(pm_A)}}}{1+\Delta} + \sqrt{4\sigma^2\log(p)}\\
& \le (b+2\sigma) [(\sqrt{\log{p(n-m_A)}}n^{-\frac{1+\delta}{3}}\sqrt{m_A}) \vee \frac{\sqrt{\log(pm_A)}}{1+\Delta} \vee \log(p)]
\end{align*}
with probability at least $1-\exp(-c\log(p(n-m_A)))-\exp(-c'\log(pm_A))$ for sufficiently large $n$.\newline

For the column-wise $L_{\infty}$ norm, we have
\begin{align}\label{eq:L1_Bound1}
  \max_j &  \norm{\ical^j}_1 = \max_j\sum_i \frac{\alpha\tau_i|B_{ij}|}{1+\alpha\tau_i}
  \le \max_j \Big( \sum_{i\le n-m_A}|B_{ij}| + \sum_{i>n-m_A}\frac{\alpha\tau_i|B_{ij}|}{1+\alpha\tau_i} \Big)\notag\\
& \le  \max_j \Big( b\frac{n-m_A}{\tau_{n-m_A}}n^{-\frac{1+\delta}{3}} + b\sum_{i>n-m_A}\frac{\alpha}{1+\alpha\tau_{n-1}} n^{-\frac{1+\delta}{3}} \Big)  ~~\text{|~by Assumption~\ref{ass:regressionCoef}~|} \notag\\
& \le  \max_jb\Big(\frac{n-m_A}{\tau_{n-m_A}}n^{-\frac{1+\delta}{3}} + \sum_{n-m_A<i < n}\frac{\alpha n^{-\frac{1+\delta}{3}}}{1+\Delta} + \alpha n^{-\frac{1+\delta}{3}}\Big) \notag\\
&    =  b((n-m_A)\sqrt{m_A} n^{-\frac{1+\delta}{3}} + \frac{m_A}{1+\Delta}+1).
\end{align}
For the second term, 
\begin{align}\label{eq:L1_Bound2}
\max_j\norm{\iical^j}_1 &\le \max_j\sum_{i\le n-m_A} \frac{1}{1+\alpha\tau_i}|\tilde{E}_{ij}| +  \max_j\sum_{i>n-m_A} \frac{1}{1+\alpha\tau_i}|\tilde{E}_{ij}\notag|\\
& \le  \frac{1}{1+\tau_{n-m_A}n^{\frac{1+\delta}{3}}}\max_j\sum_{i\le n-m_A}|\tilde{E}_{ij}| + \max_j\sum_{n-m_A<i < n} \frac{|\tilde{E}_{ij}|}{1+\Delta} + \max_j|\tilde{E}_{nj}|.
\end{align}
By Lemma~\ref{lem:concentration}, for each $j\in [p]$, $\p(\sum_{i\le n-m_A}|\tilde{E}_{ij}| > 2\sigma(n-m_A)) \le \exp(-2c(n-m_A))$ for some constant $c$; therefore 
$$\p(\max_j\sum_{i\le n-m_A}|\tilde{E}_{ij}| > 2\sigma(n-m_A)) \le p\exp(-2c(n-m_A)) \le \exp(-cn), $$
as long as $\log{p} = o(n)$ and $m_A = o(n)$.

Assume $m_A\ge 2$, again by Lemma~\ref{lem:concentration}, for each $j\in [p]$, 
$$\p(\sum_{n-m_A<i<n} |\tilde{E}_{ij}|>2\sigma (m_A-1)\sqrt{c'\log{p}}) \le \exp(-2Cm_A\log{p})$$
 for some constant $C,c'>0$ with $C>1$. Therefore, 
$$\p(\max_j\sum_{n-m_A<i<n} |\tilde{E}_{ij}|>2\sigma m_A\sqrt{c'\log{p}}) \le p\exp(-2Cm_A\log{p}) \le \exp(-C(m_A-1)\log{p})$$
and
$$\p(\max_j|\tilde{E}_{nj}| > 2\sigma\sqrt{c'\log{p}}) \le \exp(-C\log{p}).$$
The above result is also trivially true if $m_A=1$. Substituting these two inequalities into \eqref{eq:L1_Bound2} gives
\begin{align}\label{eq:L1_Bound3}
\max_j\norm{\iical^j}_1 & \le  \frac{2\sigma(n-m_A)}{\tau_{n-m_A}n^{\frac{1+\delta}{3}}} + 2\sigma (\frac{m_A}{\Delta+1}+1)\sqrt{c'\log{p}}\notag\\
& \le 2\sigma((n-m_A)\sqrt{m_A}n^{-\frac{1+\delta}{3}}+(\frac{m_A}{\Delta+1}+1)\sqrt{c'\log{p}})
\end{align}
with probability at least $1-\exp(-cn)-\exp(-Cm_A\log{p})-\exp(-C\log{p})$.   Now combining \eqref{eq:L1_Bound1} and \eqref{eq:L1_Bound3}, we get
\begin{align*}
\norm{\hat{B}-B}_{1,1} & \le \max_j\norm{\ical^j}_1 + \max_j\norm{\iical^j}_1 \\
&\le b((n-m_A)\sqrt{m_A} n^{-\frac{1+\delta}{3}} + (\frac{m_A}{\Delta+1}+1))   + 2\sigma((n-m_A)\sqrt{m_A}n^{-\frac{1+\delta}{3}}+(\frac{m_A}{\Delta+1}+1)\sqrt{c'\log{p}}) \\
& \le (b+2\sigma)(\sqrt{m_A}n^{\frac{2-\delta}{3}}+(\frac{m_A}{\Delta+1}+1)(1\vee \sqrt{c'\log{p}}))\\
&\le (1\vee \sqrt{c'})(b+2\sigma)(\sqrt{m_A}n^{\frac{2-\delta}{3}}+(\frac{m_A}{\Delta+1}+1)\sqrt{\log{p}}).
\end{align*}
with probability at least $1-\exp(-cn)-\exp(-Cm_A\log{p})-\exp(-C\log{p})$ as long as $p\ge 3$.

Finally, for the Frobenius norm we have
\begin{align}\label{eq:L2_Bound1}
\sum_j\norm{\ical^j}_2^2 &= \sum_j\sum_i \frac{\alpha^2\tau_i^2|B_{ij}|^2}{(1+\alpha\tau_i)^2} \le \sum_j\Big(\sum_{i\le n-m_A}|B_{ij}|^2 +\sum_{i>n-m_A}\frac{\alpha^2\tau_i^2|B_{ij}|^2}{(1+\alpha\tau_i)^2}\Big)\notag\\
& \le b^2\sum_j\Big(\frac{n-m_A}{\tau_{n-m_A}^2}n^{-\frac{2(1+\delta)}{3}}+ \sum_{i>n-m_A}(\frac{\alpha}{1+\alpha\tau_{n-1}})^2n^{-\frac{2(1+\delta)}{3}}\Big) ~~\text{|~by Assumption~\ref{ass:regressionCoef}~|} \notag\notag\\
& \le  b^2p\Big((n-m_A)m_An^{-\frac{2(1+\delta)}{3}} + \frac{m_A}{(1+\Delta)^2} + 1\Big).
\end{align}
For the second term, 
\begin{align}\label{eq:L2_Bound2}
\sum_j\norm{\iical^j}_2^2 &= \sum_j\Big(\sum_{i\le n-m_A} (\frac{1}{1+\alpha\tau_i})^2|\tilde{E}_{ij}|^2 +  \sum_{i>n-m_A} (\frac{1}{1+\alpha\tau_i})^2|\tilde{E}_{ij}|^2\Big)\notag\\
& \le \frac{1}{\tau_{n-m_A}^2}n^{-\frac{2(1+\delta)}{3}}\sum_{i\le n-m_A} \sum_j|\tilde{E}_{ij}|^2  + \sum_{n-m_A<i <n} \sum_j|\tilde{E}_{ij}|^2/(1+\Delta)^2 + \sum_j|\tilde{E}_{nj}|^2\notag\\
&\le m_An^{-\frac{2(1+\delta)}{3}}\sum_{i\le n-m_A} \norm{\tilde{E}_{i\cdot}}_2^2  + \sum_{n-m_A<i <n} \norm{\tilde{E}_{i\cdot}}_2^2/(1+\Delta)^2 + \norm{\tilde{E}_{n\cdot}}^2.
\end{align}
If $m_A\ge 2$, by \eqref{eq:4thConcentration} from Lemma~\ref{lem:concentration}, for a proper $c$, we have
\begin{align*}
  \p(\sum_{i\le n-m_A} \norm{\tilde{E}_{i\cdot}}_2^2 > 2(n-m_A)p\sigma^2) & \le \p(\sum_{i\le n-m_A} \norm{\tilde{E}_{i\cdot}}_2^2 > 2(n-m_A)\tr(\Sigma)) \le 2\exp(-c(n-m_A)r(\Sigma)) , \\
  \p(\sum_{n-m_A<i <n} \norm{\tilde{E}_{i\cdot}}_2^2 > 2m_Ap\sigma^2) &\le \p(\sum_{i> n-m_A} \norm{\tilde{E}_{i\cdot}}_2^2 > 2m_A\tr(\Sigma)) \le 2\exp(-cm_Ar(\Sigma)),\\
  \p(\norm{\tilde{E}_{n\cdot}}^2 > 2p\sigma^2) &\le 2\exp(-c\frac{p\sigma^2}{\phi_{max}(\Sigma)}).
  \end{align*}
Putting everything together, 
\begin{align*}
\norm{\hat{B}-B}_F^2 &\le \sum_j\norm{\ical^j}_2^2+\sum_j\norm{\iical^j}_2^2\\
& \le b^2p\Big((n-m_A)m_An^{-\frac{2(1+\delta)}{3}} + \frac{m_A}{(1+\Delta)^2}+1\Big)\\
&~~~~~~~~~~~~~~ + 2(n-m_A)pm_An^{-\frac{2(1+\delta)}{3}}\sigma^2  + 2(\frac{m_A}{(1+\Delta)^2}+1)p\sigma^2\\
& = (b^2+2\sigma^2)p\Big((n-m_A)m_An^{-\frac{2(1+\delta)}{3}} + \frac{m_A}{(1+\Delta)^2}+1\Big)
\end{align*}
with probability at least $1-2\exp(-c(n-m_A)r(\Sigma))-2\exp(-cm_Ar(\Sigma))- 2\exp(-c\frac{p\sigma^2}{\phi_{max}(\Sigma)})$. Notice that when $m_A=1$, the above result is trivially true.

\end{proof}

\begin{proof}[Proof of Theorem~\ref{thm:MeanInitialError}]
  By definition, we have $\norm{\hat{M}-M}_F = \norm{U(\hat{B}-B)}_F = \norm{\hat{B}-B}_F$.   Thus the theorem follows directly from Lemma~\ref{lem:meanCoefbound} and the fact that $n-m_A\le n$.    Note that the Frobenius norm bound in Lemma~\ref{lem:meanCoefbound} does not need $\log{p} = o(n)$ and $m_A=o(n)$.
\end{proof}


Now we proceed to prove Theorem~\ref{thm:two-stageMaxBound}. Let 
\begin{align*}
  \hat{S} = \frac{1}{n}(X-\hat{M})^T(X-\hat{M}) \\
  S = \frac{1}{n}(X-M)^T(X-M)
  \end{align*}
$S$ is the sample covariance matrix used by the glasso algorithm when the mean is assumed known (and without loss of generality set to 0).   The success of glasso is dependent on $S$ concentrating around the true covariance matrix $\Sigma$.    If we can show $\hat S$ concentrates around $\Sigma$, we should be able to prove similar properties of GNC-lasso.

\begin{lem}\label{lem:Sconcentration}
Under the conditions of Theorem~\ref{thm:MeanInitialError} and assuming $\log{p} = o(n), m_A = o(n)$, we have
\begin{align*}
\norm{\hat{S}-\Sigma}_{\infty} \le C\max\Big(& \sqrt{\log{(pn)}}m_An^{-\frac{2+2\delta}{3}}  ,  \sqrt{\log{(pn)}}\sqrt{\log p}m_A^{3/2}n^{-\frac{4+\delta}{3}},\notag\\
&\sqrt{\log{(pn)}}\sqrt{m_A}n^{-\frac{1+\delta}{3}} ,  \sqrt{\log{(pn)}}\sqrt{\log p}\frac{m_A}{n}      ,     \sqrt{\frac{\log p}{n}}              \Big)
\end{align*}
with probability at least $1-\exp(-c\log(p(n-m_A))) - \exp(-c\log(pm_A)) - \exp(-c\log{p})$ for some constant $C$ and $ c$  that only depend on $b$ and $\sigma$.
\end{lem}
\begin{proof}[Proof of Lemma~\ref{lem:Sconcentration}]
We will be using $C$ and $c$ to denote generic constants whose value might change across different lines. Using the triangular inequality, 
$$\norm{\hat{S}-\Sigma}_{\infty} \le \norm{\hat{S}-S}_{\infty} + \norm{S-\Sigma}_{\infty}.$$
we will prove concentration in two steps.   Starting with the first term and writing $X=M+E$, we have 
\begin{align}\label{eq:ShatDecomposition}
\hat{S} - S& = \frac{1}{n}(UB+E-U\hat{B})^T(UB+E-U\hat{B}) - \frac{1}{n}E^TE\notag\\
& = \frac{1}{n}[(\hat{B}-B)^T(\hat{B}-B)-E^TU((\hat{B}-B))-((\hat{B}-B))^TU^TE + E^TE] - \frac{1}{n}E^TE\notag\\
& = \frac{1}{n}(\hat{B}-B)^T(\hat{B}-B)-\frac{1}{n}E^TU(\hat{B}-B)-\frac{1}{n}(\hat{B}-B)^TU^TE.
\end{align}
By Lemma~\ref{lem:meanCoefbound}, for some constant $C$ depending on $b$ and $\sigma$, 
\begin{align}\label{eq:ShatBound1}
\norm{\frac{1}{n}(\hat{B}-B)^T(\hat{B}-B)}_{\infty} & \le \frac{1}{n}\norm{\hat{B}-B}_{\infty}\norm{\hat{B}-B}_{1,1} \notag\\
& \le \frac{C}{n} [(\sqrt{\log{p(n-m_A)}}n^{-\frac{1+\delta}{3}}\sqrt{m_A}) \vee \frac{\sqrt{\log(pm_A)}}{1+\Delta} \vee \sqrt{\log(p)}] \times \notag\\
&~~~~~~~~~~~~~~~~~~~~~~~[\sqrt{m_A} n^{\frac{2-\delta}{3}}\vee \sqrt{\log{p}}(\frac{m_A}{\Delta+1}+1)]\notag\\
& \le C\max\Big(  \sqrt{\log{(pn)}}m_An^{-\frac{2+2\delta}{3}}  ,  \sqrt{\log{(pn)}}\sqrt{\log p}n^{-\frac{4+\delta}{3}}\sqrt{m_A}(\frac{m_A}{\Delta+1}+1),\notag\\
&~~~~~~~~~~~~~  \frac{\sqrt{\log{(pm_A)}}\sqrt{m_A}n^{-\frac{1+\delta}{3}}}{1+\Delta} ,  \frac{\sqrt{\log{(pm_A)}}\sqrt{\log p}}{(1+\Delta)n}(\frac{m_A}{\Delta+1}+1) ,\notag\\
&~~~~~~~~~~~~~       \sqrt{\log{p}}\sqrt{m_A}n^{-\frac{1+\delta}{3}},     \frac{\log{p}}{n}(\frac{m_A}{\Delta+1}+1)   \Big)
\end{align}
with probability at least $1-\exp(-c\log(p(n-m_A))) - \exp(-c\log(pm_A)) - \exp(-c\log(p))$.

On the other hand, note that $U^TE = (U_{\i\cdot}E_{\cdot j})_{i,j=1}^n$ and $\norm{U_{i\cdot}}_2 = 1$, so $(U^TE)_{ij} \sim \ncal(0,\sigma^2)$.   Therefore 
$$\norm{U^TE}_{\infty} \le \sqrt{2\sigma^2\log(np)}$$
with probability at least $1-\exp(-c\log(np))$.    Hence the second and third terms in \eqref{eq:ShatDecomposition} satisfy
\begin{align}\label{eq:ShatBound2}
\norm{\frac{1}{n}E^TU(\hat{B}-B)}_{\infty}& \le \frac{1}{n}\norm{U^TE}_{\infty}\norm{\hat{B}-B}_{1,1}\notag\\
& \le C\frac{1}{n}\sqrt{\log(np)}[\sqrt{m_A} n^{\frac{2-\delta}{3}}\vee \sqrt{\log{p}}(\frac{m_A}{\Delta+1}+1)]\notag \\
& = C [\sqrt{\log{(np)}}\sqrt{m_A}n^{-\frac{1+\delta}{3}} + \frac{\sqrt{\log{(np)}}\sqrt{\log p}}{n}(\frac{m_A}{\Delta+1}+1)]
\end{align}
with probability at least $1-\exp(-c\log(p(n-m_A)) - \exp(-c\log(pm_A)) -\exp(-c\log(np))$. Note that both terms in \eqref{eq:ShatBound2} dominate the last two terms in \eqref{eq:ShatBound1}. Thus substituting \eqref{eq:ShatBound1} and \eqref{eq:ShatBound2} into \eqref{eq:ShatDecomposition} leads to
\begin{align*}
\norm{\hat{S} - S}_{\infty} &\le C\max\Big(  \sqrt{\log{(pn)}}m_An^{-\frac{2+2\delta}{3}}  ,  \sqrt{\log{(pn)}}\sqrt{\log p}n^{-\frac{4+\delta}{3}}\sqrt{m_A}(\frac{m_A}{\Delta+1}+1),\notag\\
&~~~~~~~~~~~~~ \sqrt{\log{(np)}}\sqrt{m_A}n^{-\frac{1+\delta}{3}} ,  \frac{\sqrt{\log{(np)}}\sqrt{\log p}}{n}(\frac{m_A}{\Delta+1}+1) ,\notag\\
&~~~~~~~~~~~~~       \sqrt{\log{p}}\sqrt{m_A}n^{-\frac{1+\delta}{3}},     \frac{\log{p}}{n}(\frac{m_A}{\Delta+1}+1)   \Big)
\end{align*}
with probability at least $1-\exp(-c\log(p(n-m_A))) - \exp(-c\log(pm_A))- \exp(-c\log(p))$.

In addition, Lemma 1 of \cite{ravikumar2011high} implies that 
$$\norm{S-\Sigma}_{\infty} \le  \sqrt{\frac{2c\log p}{n}}$$
with probability at least $1-\exp(-c\log{p})$. Therefore, we have
\begin{align*}
\norm{\hat{S}-\Sigma}_{\infty} &\le C\max\Big(  \sqrt{\log{(pn)}}m_An^{-\frac{2+2\delta}{3}}  ,  \sqrt{\log{(pn)}}\sqrt{\log p}n^{-\frac{4+\delta}{3}}\sqrt{m_A}(\frac{m_A}{\Delta+1}+1),\notag\\
&~~~~~~~~~~~~~ \sqrt{\log{(np)}}\sqrt{m_A}n^{-\frac{1+\delta}{3}} ,  \frac{\sqrt{\log{(np)}}\sqrt{\log p}}{n}(\frac{m_A}{\Delta+1}+1) ,\notag\\
&~~~~~~~~~~~~~       \sqrt{\log{p}}\sqrt{m_A}n^{-\frac{1+\delta}{3}},     \frac{\log{p}}{n}(\frac{m_A}{\Delta+1}+1) ,\sqrt{\frac{\log p}{n}}  \Big)
\end{align*}
with probability at least $1-\exp(-c\log(p(n-m_A))) - \exp(-c\log(pm_A))-2\exp(-c\log{p})$.

\end{proof}

For conciseness, we present Theorem~\ref{thm:two-stageMaxBound} in the main text by assuming the more interesting situations $p \ge  n^{c_0}$. This is not necessary in any sense, so here we prove a trivially more general version of Theorem~\ref{thm:two-stageMaxBound} without the high-dimensional assumption. For completeness, we rewrite the theorem here.

\begin{thm}[Trivially generalized version of Theorem~\ref{thm:two-stageMaxBound}]\label{thm:two-stageMaxBound'}
Under the conditions of Theorem~\ref{thm:MeanInitialError} and Assumption~\ref{ass:irrepresentability}, if  $\log{p} = o(n)$ and $m_A= o(n)$, there exist some positive constants $C,c, c', c''$ that only depend on $b$ and $\sigma$, such that if $\hat{\Theta}$ is the output of Algorithm~\ref{algo:twostage} with $\alpha = n^{\frac{1+\delta}{3}}$, $\lambda = \frac{8}{\rho}\nu(n,p)$ where
\begin{align}\label{eq:nu'}
\nu(n,p) &:= C\max\Big(  \sqrt{\log{(pn)}}m_An^{-\frac{2+2\delta}{3}}  ,  \sqrt{\log{(pn)}}\sqrt{\log p}n^{-\frac{4+\delta}{3}}\sqrt{m_A}(\frac{m_A}{\Delta+1}+1),\notag\\
&~~~~~~~~~~~~~ \sqrt{\log{(np)}}\sqrt{m_A}n^{-\frac{1+\delta}{3}} ,  \frac{\sqrt{\log{(np)}}\sqrt{\log p}}{n}(\frac{m_A}{\Delta+1}+1) ,\notag\\
&~~~~~~~~~~~~~       \sqrt{\log{p}}\sqrt{m_A}n^{-\frac{1+\delta}{3}},     \frac{\log{p}}{n}(\frac{m_A}{\Delta+1}+1) ,\sqrt{\frac{\log p}{n}}  \Big)
\end{align}
and $n$ sufficiently large so that  
 $$\nu(n,p)<\frac{1}{6(1+8/\rho)\psi\max\{\kappa_{\Sigma}\kappa_{\Gamma},  (1+8/\rho)\kappa_{\Sigma}^3\kappa_{\Gamma}^2   \}},$$
 then with probability at least $1-\exp(-c\log(p(n-m_A))) - \exp(-c'\log(pm_A)) - \exp(-c''\log{p})$,  then the estimate $\hat{\Theta}$ has the following properties: 
\begin{enumerate}
\item  Error bounds:  
\begin{eqnarray*}
\norm{\hat{\Theta}-\Theta}_{\infty} & \le &  2(1+8/\rho)\kappa_{\Gamma}\nu(n,p)
  \\
   \norm{\hat{\Theta}-\Theta}_F & \le  & 2(1+8/\rho)\kappa_{\Gamma}\nu(n,p)\sqrt{s+p}.  \label{eq:FrobeniusErrorBound}\\
   \norm{\hat{\Theta}-\Theta}~~ &\le& 2(1+8/\rho)\kappa_{\Gamma}\nu(n,p)\min(\sqrt{s+p}, \psi).    
 \end{eqnarray*}

 \item Support recovery: 
   $$S(\hat{\Theta}) \subset S(\Theta), $$
   and if additionally 
$\min_{(j,j')\in S(\Theta)}|\Theta_{jj'}| > 2(1+8/\rho)\kappa_{\Gamma}\nu(n,p),$
then $$S(\hat{\Theta}) = S(\Theta).$$
%
\end{enumerate}
\end{thm}

We will use the {\em primal-dual witness} strategy from \cite{ravikumar2011high} for proof.   We show that even if $\hat{S}$ has a worse concentration around $\Sigma$ than $S$, we can still achieve consistency and sparsistency under certain regularity conditions.

\begin{proof}[Proof of  Theorem~\ref{thm:two-stageMaxBound'} and Theorem~\ref{thm:two-stageMaxBound}]
The argument follows the proof of Theorem 1 in \cite{ravikumar2011high}. In particular, for the event where the bound in Lemma~\ref{lem:Sconcentration} holds, we just have to show that the primal-dual witness construction succeeds. The choice of $\lambda = \frac{8}{\rho}\nu(n,p)$ ensures $\norm{\hat{S} - \Sigma}_{\infty} \le \frac{\rho}{8}\lambda$.  With the requirement on the sample size, the assumptions of Lemma 5 and 6 in \cite{ravikumar2011high} hold, implying strict dual feasibility holds for the primal-dual witness, which shows the procedure succeeds. Then the first claim of the theorem is a direct result of Lemma 6 in \cite{ravikumar2011high} and the second claim is true by construction of the primal-dual witness procedure. The remaining bounds can be proved similarly.

Finally, if $p \ge n^{c_0}$, we have $\log(p) \ge c_0 \log (n)$, so in \eqref{eq:nu'}, the 3rd  coincides with the 5th, and the $4$th term coincides with the 6th term, by magnitude, resulting in the form
\begin{align*}
&C\max\Big( m_A n^{-\frac{2+2\delta}{3}}  \sqrt{\log{p}},  \sqrt{m_A}n^{-\frac{4+\delta}{3}}(\frac{m_A}{\Delta+1}+1)\log{p},\notag\\
&~~~~~~~~~~~~~       \sqrt{m_A}n^{-\frac{1+\delta}{3}}\sqrt{\log{p}},     (\frac{m_A}{\Delta+1}+1) \frac{\log{p}}{n},\sqrt{\frac{\log p}{n}}  \Big)\\
 =& C\sqrt{\frac{\log p}{n}}\max\Big(1, m_A n^{-\frac{1+4\delta}{6}},  \sqrt{m_A}n^{-\frac{5+2\delta}{6}}(\frac{m_A}{\Delta+1}+1)\sqrt{\log{p}}, \sqrt{m_A}n^{\frac{1-2\delta}{6}},     (\frac{m_A}{\Delta+1}+1) \sqrt{\frac{\log{p}}{n} } \Big).
\end{align*}

To further simplified the form, we apply another upper bound for the term by the fact that $\sqrt{m_A} \ge 1$, $\frac{m_A}{\Delta+1}+1 \ge 1$ and $\delta \ge 0$, which gives

\begin{align*}
 & C\sqrt{\frac{\log p}{n}}\max\Big(1, m_A n^{-\frac{1+4\delta}{6}},  \sqrt{m_A}n^{-\frac{5+2\delta}{6}}(\frac{m_A}{\Delta+1}+1)\sqrt{\log{p}}, \sqrt{m_A}n^{\frac{1-2\delta}{6}},     (\frac{m_A}{\Delta+1}+1) \sqrt{\frac{\log{p}}{n} } \Big)\\
 & \le C\sqrt{\frac{\log p}{n}}\max\Big(1, m_A n^{-\frac{1+4\delta}{6}},  \sqrt{m_A}n^{-\frac{5+2\delta}{6}}(\frac{m_A}{\Delta+1}+1)\sqrt{\log{p}} +    (\frac{m_A}{\Delta+1}+1) \sqrt{\frac{\log{p}}{n} }, \sqrt{m_A}n^{\frac{1-2\delta}{6}} \Big)\\
 & \le C\sqrt{\frac{\log p}{n}}\max\Big(1, m_A n^{-\frac{1+4\delta}{6}}, \sqrt{m_A}n^{\frac{1-2\delta}{6}} , \sqrt{\frac{\log{p}}{n} }(\frac{m_A}{\Delta+1}+1)(\sqrt{m_A}n^{-\frac{1+\delta}{3}} +   1 ) \Big).
\end{align*}

\end{proof}

\section{Oracle mean estimation by GNC-lasso}\label{sec:oracle}

In our setting, unlike in the classical glasso setting, the mean estimate is also of interest, and in this section we show that our estimate $\hat{M}$ enjoys a weak oracle property in a certain sense.    We use the spectrum of $\lcal_s$  as a basis again, writing $U$ for the matrix of eigenvectors of $\lcal_s$ and expanding a matrix $M \in \bR^{n\times p}$  as $M = UB$.  Since $U$ is given and orthonormal, estimating $M$ is equivalent to estimating $B$. 
In an ideal scenario,  if the true value $\Theta$ is given to us by an oracle, we could estimate $B$ by minimizing one of the two objective functions: 
\begin{align}
  \min_{B\in \bR^{n\times p}}& \tr((X-UB)^T(X-UB)) + \alpha \, \tr(B^T\Lambda B),
  \label{obj:MeanEstimation1} \\
  \min_{B\in \bR^{n\times p}}&  \tr(\Theta(X-UB)^T(X-UB)) + \alpha \, \tr(B^T\Lambda B),
\label{obj:MeanEstimation2}
\end{align}
where $\Lambda = \F{diag}(\tau_1, \tau_2, \cdots, \tau_n)$ is the diagonal matrix of eigenvalues of $\lcal_s$. It is easy to verify that \eqref{obj:MeanEstimation1} is equivalent to the mean estimation step \eqref{eq:individualRNC} in the two-stage procedure (up to $U$), while \eqref{obj:MeanEstimation2} is equivalent to estimating the mean by maximizing the joint penalized likelihood \eqref{eq:obj1} with $\Theta$ fixed at the true value.   We can then treat \eqref{obj:MeanEstimation2} as an oracle estimate in the sense that it uses the true value of the covariance matrix.    It serves as a benchmark for the best performance one could expect in estimating $B$ (or equivalently $M$). Let $\hat{B}_1$ and $\hat{B}_2$ be the estimates from \eqref{obj:MeanEstimation1} and \eqref{obj:MeanEstimation2}, respectively,  and let $W_k= B-\hat{B}_k$, $k=1,2$ be the corresponding estimation error matrices. We then have the following result.
\begin{prop}\label{thm:oracle}
Under model \eqref{eq:GNC-GGM}, assume $W_1$ and $W_2$ are the errors defined above with the same tuning parameter $\alpha$. Under the Assumption~\ref{ass:spectralBound}, if $\Theta$ is diagonally dominant with $\max_j\frac{\sum_{j'\ne j}|\Theta_{j'j}|}{\Theta_{jj}} \le \rho < 1$, then there exist a matrix $\tilde{W}$ such that
$$
(1-\rho)\frac{1}{\bar{k}}\le \frac{\norm{\tilde{W}}_{\infty}}{\norm{W_2}_{\infty}} \le (1+\rho)\bar{k}
$$
for the constant $\bar{k}$ in Assumption~\ref{ass:spectralBound} and
$$W_1-\tilde{W} = (I+\alpha \Lambda)^{-1}U^TE(I-\Theta).$$
where each row $E$ is i.i.d from multivariate Gaussian $\ncal(0, \Sigma).$
\end{prop}
Proposition~\ref{thm:oracle} shows $\tilde{W}$ and $W_1$ are stochastically equivalent while $\tilde{W}$ and $W_2$ are roughly the same in $\norm{\cdot}_{\infty}$. Therefore, \eqref{obj:MeanEstimation1} and \eqref{obj:MeanEstimation2} are essentially equivalent in the sense of entrywise error bound, implying that $\hat{M}$ computed by GNC-lasso cannot be non-trivially improved by the oracle estimator under the true model with known $\Theta$. 

Proposition~\ref{thm:oracle} makes an additional assumption on diagonal dominance of $\Theta$, which is a relatively mild assumption consistent with others in this context.   To see this, consider a general multivariate Gaussian vector $y \sim \ncal(0, \Sigma)$.  Then we can write  
$$y_j = \sum_{j'\ne j} \zeta^{j}_{j'} y_{j'} + \xi_j$$
where the vector $\zeta^{j} \in \bR^{p}$ satisfies $\zeta^{j}_{j'} = -\frac{\Theta_{jj'}}{\Theta_{j'j'}}$ for $j' \ne j$ and $\zeta^{j}_j = 0$, and $\xi_j$ is a Gaussian random variable with zero mean and variance equal to the  conditional variance of $y_j$ given $\{y_{j'}\}_{j'\ne j}$.    Thus the diagonal dominance assumption of Proposition~\ref{thm:oracle}  is essentially assuming 
$$\max_j\norm{\zeta^j}_1 = \max_j\sum_{j'\ne j}|\zeta^{j}_{j'}| < \rho <1.$$ 
This has the same form as Assumption 4 of \cite{meinshausen2006high}, who proposed node-wise regression to estimate the Gaussian graphical model.   There $\rho < 1$ is needed for node-wise regression to consistently estimate the graph structure (see Proposition 4 of \cite{meinshausen2006high}). 

\begin{rem}[Implications for iterative estimation]
If the iterative algorithm is used to obtain $\tilde{M}$ and $\tilde{\Theta}$, we know $\tilde{M}$ is the solution of \eqref{obj:MeanEstimation2} with $\Theta$ replaced by $\tilde{\Theta}$. Since $\tilde{\Theta}$ is only an estimate of $\Theta$, we would not expect this estimator to work as well as the oracle  estimator \eqref{obj:MeanEstimation2}.   Since $\hat{M}$ cannot be improved by the oracle estimator, intuitively we make the conjecture that $\tilde{M}$ cannot significantly improve on $\hat{M}$ either.   
\end{rem}

To prove Proposition~\ref{thm:oracle}, we need a few properties of Kronecker products. Recall that given two matrices $A\in \bR^{m\times n}$ and $B\in \bR^{p\times q}$, their Kronecker product is defined to be an $(mp)\times (nq)$ matrix such that
$$
A\otimes B = \begin{pmatrix}
  A_{11}B & A_{12}B & \cdots & A_{1n}B \\
  A_{21}B & A_{22}B & \cdots & A_{2n}B \\
  \vdots  & \vdots  & \ddots & \vdots  \\
  A_{m1}B & A_{m2}B & \cdots & A_{mn} B
 \end{pmatrix}.
$$
For a matrix $A$, define $\ve(A)$ to be the column vector stacking all columns of $A$,  $ \ve(A)= (A_{\cdot 1}, A_{\cdot 2}, \cdots, A_{\cdot n})$.  Some standard properties we'll need, assuming the matrix dimensions match appropriately, are stated next.
\begin{align*}
&\ve(AB) = (I_q\otimes A)\ve(B), A\in \bR^{n\times p}, B\in \bR^{p\times q}\\
&\ve(B^T\otimes A)\ve(C) = \ve(ACB), A\in \bR^{m\times n}, B\in \bR^{p\times q}, C \in \bR^{n\times p} \\
&(A\otimes B)(C\otimes D) = (AC)\otimes(BD)\\
&\tr(ABA^T) = \ve(A)^T(B\otimes I_n)\ve(A) \notag\\
&~~~~~~~~~~~~~~~~~~~~~~= \ve(A^T)^T(I_n\otimes B)\ve(A^T), A\in\bR^{n\times p}, B\in \bR^{p\times p}. 
\end{align*}

\begin{prop}\label{prop:meanErrorEquation}
For the estimates $W_1$ from \eqref{obj:MeanEstimation1} and $W_2$ from \eqref{obj:MeanEstimation2}, we have
\begin{align}
  W_1I_p + \alpha \Lambda W_1 = \alpha\Lambda B + \tilde{E},
  \label{eq:meanError1} \\
  W_2\Theta + \alpha \Lambda W_2 = \alpha\Lambda B + \dot{E},
    \label{eq:meanError2}
\end{align} 
where $\tilde{E} = (\tilde{\epsilon}_{1\cdot}, \tilde{\epsilon}_{2\cdot}, \cdots, \tilde{\epsilon}_{n\cdot})$ and $\tilde{\epsilon}_{i\cdot} \sim \ncal(0,\Sigma)$ are i.i.d., and  $\dot{E} = (\dot{\epsilon}_{1\cdot}, \dot{\epsilon}_{2\cdot}, \cdots, \dot{\epsilon}_{n\cdot})$,  and  $\dot{\epsilon}_{i\cdot} \sim\ncal(0,\Theta)$ are i.i.d. In particular, $\tilde{E} = -U^TE$ and $\dot{E} = -U^TE\Theta$.
\end{prop}

\begin{proof}[Proof of Proposition~\ref{prop:meanErrorEquation}]
We only prove \eqref{eq:meanError2};  the proof of \eqref{eq:meanError1} is exactly the same, with $\Theta$ replaced by $I_p$. The conclusion follows directly from writing out the quadratic optimiation solution after vectorizing all matrices. Specifically, the objective function \eqref{obj:MeanEstimation2} can be written as 
\begin{align*}
\tr (\Theta & (X-UB)^T(X-UB)) + \alpha \tr(B^T\Lambda B) = \\
&= \ve(X-UB)^T(\Theta\otimes I_n)\ve(X-UB) + \alpha\ve(B)^T(I_p\otimes \Lambda)\ve(B)\\
&= \ve(UB)^T(\Theta\otimes I_n)\ve(UB) - 2\ve(UB)^T(\Theta\otimes I_n)\ve(X) + \alpha\ve(B)^T(I_p\otimes \Lambda)\ve(B) + const\\
& = \ve(B)^T(I_p\otimes U^T)(\Theta\otimes I_n)(I_p\otimes U)\ve(B)-2\ve(X)^T(\Theta\otimes I_n)(I_p\otimes U)\ve(B)\\
&~~~~~~~~~~~~~~~~~~~~~~~~~~~~~~~~~~~~~~~~~~~~~~~~~~~~ + \alpha\ve(B)^T(I_p\otimes \Lambda)\ve(B)+const\\
&= \ve(B)^T[(\Theta\otimes I_n)+\alpha(I_p\otimes \Lambda)]\ve(B)-2\ve(X)^T(\Theta\otimes U)\ve(B)+const.
\end{align*}
The minimizer of this quadratic function satisfies 
$$[(\Theta\otimes I_n)+\alpha(I_p\otimes \Lambda)]\ve(\hat{B}) = (\Theta\otimes U^T)\ve(X).$$
Substituting $X = UB+E$ into the estimating equation gives
\begin{align*}
[(\Theta\otimes I_n)+\alpha(I_p\otimes \Lambda)]\ve(\hat{B}) &= (\Theta\otimes U^T)\ve(UB+E)\\
& = (\Theta\otimes U^T)(I_p\otimes U)\ve(B) + (\Theta\otimes U^T)\ve(E)\\
&= (\Theta\otimes I_n)\ve(B) + \ve(U^TE\Theta), 
\end{align*}
and therefore 
$$(\Theta\otimes I_n)\ve(W)  + \alpha(I_p\otimes \Lambda)\ve(W) = \alpha(I_p\otimes \Lambda)\ve(B) - \ve(U^TE\Theta).$$
We then get
$$\ve(W\Theta) + \alpha\ve(\Lambda W) = \alpha\ve(\Lambda B) -\ve(U^TE\Theta).$$
This is equivalent to \eqref{eq:meanError2} by noting that $\dot{E} = -U^TE\Theta$.
  \end{proof}

Now we show $W_1$ and $W_2$ are essentially equivalent estimation errors. Define two additional estimating equations as below:
\begin{equation}\label{eq:meanError3}
W_3I_p + \alpha \Lambda W_3 = \alpha\Lambda B + \dot{E}
\end{equation}
\begin{equation}\label{eq:meanError4}
W_4\F{diag}(\Theta) + \alpha \Lambda W_4 = \alpha\Lambda B + \dot{E}
\end{equation}
The error equation \eqref{eq:meanError3} corresponds to the situation when we carry $p$ separate Laplacian smoothing estimations. The error equation \eqref{eq:meanError3} is also from $p$ separate Laplacian smoothing but it adjusts the weight each variable to be proportional to $1/\Theta_{jj}$, which can be seen as $W_2$ approximation after ignoring off-diagonal elements of $\Theta$.      Intuitively, when the off-diagonal elements are small, $W_2$ should not be very different from $W_4$, and when the diagonal elements of $\Theta$ are similar, as in Assumption~\ref{ass:spectralBound}, $W_3$ and $W_4$ should also be similar. The following proposition formalizes this intuition under the assumption that $\Theta$ is diagonally dominant.     We can then conclude that using the true $\Theta$ in \eqref{obj:MeanEstimation2} does not really bring  improvement and $W_1$, $W_2$, $W_3$, and $W_4$ are all essentially equivalent.

\begin{prop}\label{prop:noImprovement}
Assume $W_2$, $W_3$, and $W_4$ are the estimation errors from \eqref{eq:meanError2}, \eqref{eq:meanError3} and \eqref{eq:meanError4}, respectively, with the same $\alpha$. If $\Theta$ is diagonally dominant with $\max_j\frac{\sum_{j'\ne j}|\Theta_{j'j}|}{\Theta_{jj}} \le \rho < 1$, then 
\begin{equation}\label{eq:ErrorRatio}
(1-\rho)\min(1, \min_j\Theta_{jj})\le \frac{\norm{W_3}_{\infty}}{\norm{W_2}_{\infty}} \le (1+\rho)\max(1, \max_j\Theta_{jj}).
\end{equation}
In particular, under Assumption~\ref{ass:spectralBound}, 
$$
(1-\rho)\frac{1}{\bar{k}}\le \frac{\norm{W_3}_{\infty}}{\norm{W_2}_{\infty}} \le (1+\rho)\bar{k}
$$
for a constant $\bar{k}$.
\end{prop}
\begin{proof}[Proof of Proposition~\ref{prop:noImprovement}]
Directly from the definition, we have
\begin{align*}
  W_{3,ij} & = \frac{1}{1+\alpha\tau_i}(\alpha\tau_iB_{ij}+\dot{E}_{ij}) \, , \\
  W_{4,ij} & = \frac{1}{\Theta_{jj}+\alpha\tau_i}(\alpha\tau_iB_{ij}+\dot{E}_{ij}) \, . 
\end{align*}
This implies  that for any $i$, $j$ and an arbitrary $\alpha$,
\begin{equation}\label{eq:W3-W4}
\min(1, \min_j\Theta_{jj})\le \frac{W_{3,ij}}{W_{4,ij}} = \frac{\Theta_{jj}+\alpha\tau_i}{1+\alpha\tau_i} \le \max(1, \max_j\Theta_{jj}) .
\end{equation}
We next show that under the assumption of diagonal dominance of $\Theta$, even $W_2$ cannot do much better. For each $j = 1, 2, \cdots, p$, from \eqref{eq:meanError2}, 
\begin{equation*}
W_2\Theta_{\cdot j} + \alpha W_{2,\cdot j}= (\Theta_{jj}I+\alpha\Lambda)W_{2,\cdot j} + \Theta_{jj}\sum_{i\ne j}\frac{\Theta_{ij}}{\Theta_{jj}}W_{2, \cdot i} =\alpha\Lambda B_{\cdot j} + \dot{E}_{\cdot j} \ . 
\end{equation*}
Therefore, we have
\begin{equation}\label{eq:W2_col_j}
W_{2,\cdot j} + (\Theta_{jj}I+\alpha\Lambda)^{-1}\Theta_{jj}\sum_{i\ne j}\frac{\Theta_{ij}}{\Theta_{jj}}W_{2, \cdot i}= \alpha(\Theta_{jj}I+\alpha\Lambda)^{-1}\Lambda B_{\cdot j} + (\Theta_{jj}I+\alpha\Lambda)^{-1}\dot{E}_{\cdot j} = W_{4,\cdot j}
\end{equation}
in which the last equation comes from  \eqref{eq:meanError4}. By triangle inequality, \eqref{eq:W2_col_j} leads to
\begin{equation}\label{eq:W2-W4}
\norm{W_{2,\cdot j}}_{\infty}\le  \norm{W_{4,\cdot j}}_{\infty} + \norm{(\Theta_{jj}I+\alpha\Lambda)^{-1}\Theta_{jj}\sum_{i\ne j}\frac{\Theta_{ij}}{\Theta_{jj}}W_{2, \cdot i}}_{\infty} \le \norm{W_{4,\cdot j}}_{\infty} + \sum_{i\ne j}\frac{|\Theta_{ij}|}{\Theta_{jj}}\max_{i}\norm{W_{2,\cdot i}}_{\infty}. 
\end{equation}
Taking the maximum over $j$ on both sides, we have
\begin{equation}\label{W2_upper}
\norm{W_2}_{\infty}  \le \norm{W_4}_{\infty} + \rho \norm{W_2}_{\infty}. 
\end{equation}
Similarly using triangle inequality in the other direction, we get  
$$1 - \rho \le \frac{\norm{W_4}_{\infty}}{\norm{W_2}_{\infty}} \le 1+\rho \, .$$
Combining this with \eqref{eq:W3-W4}, we get
$$(1-\rho)\min(1, \min_j\Theta_{jj})\le \frac{\norm{W_3}_{\infty}}{\norm{W_2}_{\infty}} \le (1+\rho)\max(1, \max_j\Theta_{jj}).$$
Note that \eqref{eq:W2-W4} holds if we replace $\norm{\cdot}_{\infty}$ by other norms. For example, if we take the $L_1$ norm instead, we get a similar bound in $\norm{\cdot}_{1,1}$.  \end{proof}

Now we are ready to prove Proposition~\ref{thm:oracle}.
\begin{proof}[Proof of Proposition~\ref{thm:oracle}]
By taking  $\tilde{W} = W_3$ and using the conclusion of Proposition~\ref{prop:noImprovement}, the first half of Proposition~\ref{thm:oracle} directly follows.  Subtracting \eqref{eq:meanError1} from \eqref{eq:meanError3} leads to
$$(I_n+\alpha \Lambda)(\tilde{W}-W_1)=(I_n+\alpha \Lambda)(W_3-W_1) = \dot{E} - \tilde{E} = U^TE(I-\Theta) \, , $$
and therefore 
$$\tilde{W}-W_1 = (I_n+\alpha \Lambda)^{-1}U^TE(I-\Theta).$$
\end{proof}

\end{appendix}

\end{document}